\documentclass[final]{article}

\usepackage{graphicx} 
\usepackage{comment}

\usepackage{rotating}
\usepackage{tabularx}
\usepackage{url}

\usepackage{euscript}
\usepackage{amsmath}
\usepackage{mathtools}
\usepackage{amssymb}
\usepackage{tikz-cd} 
\usetikzlibrary{fadings}
\usepackage{fancyvrb} 
\usepackage{amsthm}
\usepackage{array}

\usepackage{hyperref}
\usepackage[inline]{enumitem} 

\usepackage{cite}
\usepackage{fontawesome5}

\VerbatimFootnotes  

\newcommand\numberthis{\addtocounter{equation}{1}\tag{\theequation}}

\newcommand{\C}{\mathbb{C}}
\newcommand{\R}{\mathbb{R}}
\newcommand{\K}{\mathbb{K}}
\newcommand{\Z}{\mathbb{Z}}

\newcommand{\PP}{\mathbb{P}}

\newcommand{\Mat}{\text{Mat}}
\newcommand{\rank}{\text{rank}}
\newcommand{\Sym}{\text{Sym}}
\newcommand{\End}{\text{End}}

\newcommand{\ccr}{\star}
\newcommand{\conv}{\ast}

\newcommand{\circbd}{\mathrel{\text{\faDotCircle[regular]}}}

\newcommand{\kprod}{\otimes}

\theoremstyle{definition}
\newtheorem{definition}{Definition}[section] 
\newtheorem{example}[definition]{Example}

\newtheorem{remark}[definition]{Remark}
\AtBeginEnvironment{remark}{\setlength{\parindent}{0pt}}

\theoremstyle{plain}
\newtheorem{theorem}[definition]{Theorem}
\newtheorem{lemma}[definition]{Lemma}
\newtheorem{proposition}[definition]{Proposition}
\newtheorem{corollary}[definition]{Corollary}

\usepackage{color}

\usepackage[normalem]{ulem}
\title{The Geometry of Polynomial Group Convolutional Neural Networks}
\author{Yacoub Hendi \thanks{Uppsala University, \url{yacoub.hendi@math.uu.se}} , 
Daniel Persson \thanks{Chalmers University of Technology , \url{daniel.persson@chalmers.se}} , 
Magdalena Larfors \thanks{Uppsala University, \url{magdalena.larfors@physics.uu.se}}}
\date{September 2026}

\begin{document}
\maketitle

\begin{abstract}
\noindent We study polynomial group convolutional neural networks (PGCNNs) for an arbitrary finite group 
$G$. In particular, we introduce a new mathematical framework for PGCNNs using the language of graded group algebras. This framework yields two natural parameterizations of the architecture, based on Hadamard and Kronecker products, related by a linear map. 
We compute the dimension of the associated neuromanifold, verifying that it depends only on the number of layers and the size of the group. Furthermore, we show the general fiber of both parameterizations is trivial up to the regular group action and rescaling. Hence both parametrization maps are identifiable.
\end{abstract}

\tableofcontents

\section{Introduction}

Geometric deep learning \cite{7974879, bronstein2021geometricdeeplearninggrids, Gerken2023} is the subfield in deep learning that studies equivariant architectures with respect to a given group $G$. 
The main principle of geometric deep learning is that when training over a dataset that has a set of symmetries, we can introduce an informed bias towards equivariant and invariant models by using equivariant architectures such as convolutional layers or graph layers. 
The benefit of doing this is a reduction of sample complexity; this can be seen since an equivariant architecture has fewer parameters due to weight-sharing, and fewer parameters need fewer training data points to determine their values and, henceforth, smaller sample complexity \cite{NEURIPS2018_03c6b069}. 

A classical example of equivariant architectures is convolutional neural networks (CNN), which are equivariant under discrete translations, corresponding to the group $G=\mathbb{Z}$.
A general theory for linear equivariant layers for finite groups and compact groups is presented in \cite{pmlr-v48-cohenc16, 10.5555/3454287.3455107}.
The core idea is that enforcing equivariance on linear maps, under given representation, results in weight-sharing constraints which $G$-convolutional layers have.
Hence all linear equivariant layers are convolutional and vice-versa.
In \cite{nyholm2025equivariantnonlinearmapsneural}, this approach is extended to non-linear equivariant layers.

Apart from designing equivariant models, another direction of geometric deep learning is concerned with understanding the expressivity and training dynamics of such models.
In this article we are interested in the latter direction, and to investigate this, we apply tools from the recently emergent field of \emph{neuroalgebraic geometry} \cite{marchetti2025algebraunveilsdeeplearning}.

The main idea of neuroalgebraic geometry can be briefly understood as follows. 
Given a fixed neural network architecture $\Phi$, one obtains a parametric family of functions $\Phi_\theta$ indexed by the network parameters $\theta$.
We call the space of all realizable functions \emph{the neuromanifold} $\mathcal{M}_\Phi$ associated with the architecture $\Phi$. 
The term “manifold” is a slight misnomer, as these functional spaces usually exhibit singularities.
Given a loss function $\mathcal{L}$ and a dataset $\mathcal{D}$, then learning via stochastic gradient descent can be interpreted as a distance minimization problem over the neuromanifold to an objective function $f_{\mathcal{L}, \mathcal{D}}$ in the ambient space, that depends on both the loss function $\mathcal{L}$ and the dataset $\mathcal{D}$.
From this perspective, the geometric properties of the neuromanifold translate into statistical counterparts. 
For example, dimension and curvature translate into sample complexity and expressivity, fiber classification translates into identifiability, and so on.

However, the analysis of these neuromanifolds is often intractable due to their compositional complexity and the use of nonlinear activation functions, which in some cases even fail to be analytic (e.g. the RELU function). 
One useful simplification is to restrict to polynomial activation functions, which can be justified for three reasons. Firstly, polynomials are universal approximators by Weierstrass Approximation Theorem, which implies that neural networks with non-polynomial activation functions can be approximated by networks with polynomial activation functions.
Secondly, the neuromanifold becomes a subspace of the space of homogeneous polynomials of fixed degree and in a fixed number of variables. 
More precisely, the neuromanifold becomes a semi-algebraic set (i.e. it can be described as a solution set of polynomial equations and inequalities in its finite dimensional ambient space), which can be analyzed using tools from algebraic geometry. 
Thirdly, even though non-polynomial activations are more used in mainstream uses, neural networks with polynomial activation have performed well in many different applications (e.g healthcare, finance, computer vision, cryptography, theoretical physics) \cite{DEY2016130, Fong2022, GHAZALI20113765, Nayak2018, 9665897, poly2021, DBLP:journals/corr/abs-1711-05189, pmlr-v145-douglas22a}. 

In this paper, we focus on the study of polynomial group convolutional neural networks (PGCNN) i.e. neural networks with $G$-convolutional layers and polynomial activation functions.
We formalize the construction of PGCNNs using the language of group algebras and describe the neuromanifolds as subspaces of these algebras.
The aim of this paper is to generalize the results in \cite{shahverdi2024geometryoptimizationpolynomialconvolutional} which hold for polynomial CNNs to PGCNN for an arbitrary finite group $G$.

\begin{figure}
    \centering
    \includegraphics[width=0.5\linewidth]{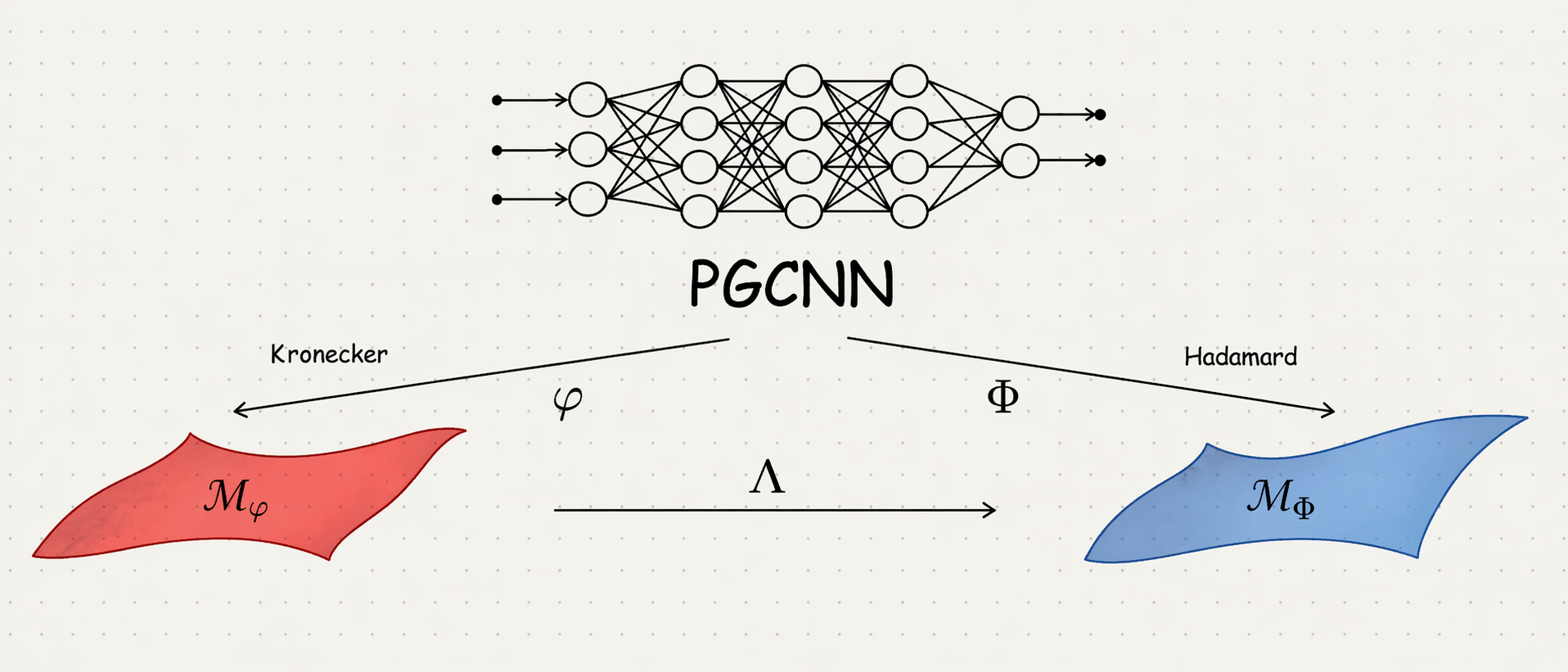}
    \caption{
    PGCNN parametrization maps $\varphi$ and $\Phi$ \protect\footnotemark.
    }
    \label{fig:pgcnn_parametrization}
\end{figure}
\footnotetext{The figure is generated by Gemini Pro.}

\paragraph{Our contribution.}

\begin{enumerate}
    \item  We propose a mathematical framework for polynomial group convolutional networks in the language of graded group algebras. See §\ref{sec: preliminaries} and §\ref{sec: pgcnn}.

    \item
    More importantly, we provide two ways to describe a polynomial activation function of degree $r$ in this language. 
    It can either be described by the Kronecker product or the Hadamard product.
    Depending on which we choose, we get two different parametrization maps denoted by $\varphi$ and $\Phi$, respectively.
    These two maps are related by $\Phi = \Lambda \circ \varphi$ where $\Lambda$ is a linear map; see Figure \ref{fig:pgcnn_parametrization}.

    \item For the polynomial group convolutional architecture of depth $L\geq 1$ and a polynomial activation function  of degree $r\geq 2$ associated with a finite group $G$, the neuromanifold $\mathcal{M}_\Phi$ and the image $\mathcal{M}_\varphi$ of $\varphi$ have the same dimension $L(|G|-1)+1$; see Theorem \ref{thm: tilde_Phi_is_regular} and Theorem \ref{thm: tilde_varphi_is_regular}.
    Note that the dimension does not depend on the degree $r$ or the particular group structure of $G$ but only on the size of $G$ and the number of layers $L$.

    \item From the previous point, it follows directly that the general fibers of the parametrization maps $\Phi$ and $\varphi$ are finite, up to scaling. Moreover, we give a precise description of the general fibers of $\varphi$ and $\Phi$, see Theorems \ref{thm: general_fiber_of_varphi}, and \ref{thm: general_fiber_of_Phi}.
    Theorem \ref{thm: general_fiber_of_Phi} shows that a PGCNN is \emph{globally identifiable} for degree $r\geq 2$ and any depth $L$.  

    \item Finally, we provide a package \texttt{PGCNNGeometry}\footnote{\url{https://github.com/jake997/PGCNNGeometry}} that uses \texttt{Sage} and \texttt{Macaulay2} to compute the dimension and the size of the general fiber of $\varphi$ and $\Phi$ for an arbitrary finite group.

\end{enumerate}

It is worth highlighting here that the results from 
\cite{massarenti2025alexanderhirschowitztheoremneurovarieties, usevich2025identifiabilitydeeppolynomialneural, finkel2025activationdegreethresholdsexpressiveness,  craciun2025linearindependencepowerspolynomials} 
show that the neuromanifolds of polynomial neural networks (PNN) with fixed width have the expected (maximum) dimension. More precisely, for a \emph{general element} of such a neuromanifold, the tangent space has the expected dimension.
Since a PGCNN can be represented as a PNN with fixed width, then one might expect that the dimension of the neuromanifold of PGCNN can be inferred immediately from these results. 
This is not true because a PGCNN enforces weight-sharing  constraints that violate the required generality assumption. In terms of algebraic geometry, the neuromanifold of PGCNN is a specialization of the neuromanifold of PNN with fixed width, and general conditions do not necessarily hold over specializations.

\paragraph{Related works.}

We review here works related to neuroalgebraic geometry and focus on those concerning convolutional and equivariant architectures.

Mehta et al. \cite{mehta2018losssurfacedeeplinear} presented one of the first attempts to use algebraic geometry to count critical points in the loss landscape for deep linear neural networks. 
Building on this perspective, Traget et al.  \cite{trager2020purespuriouscriticalpoints} introduced the classification of critical points as spurious (due to parametrization) and pure (due to geometry).
In \cite{DBLP:journals/corr/abs-1905-12207}, Kileel et al. connected the dimension of the neuromanifold of PNNs to the expressivity of the model and established upper and lower bounds on layer width for neuromanifolds to be filling (i.e. they fill their ambient space).
In the same paper, they conjectured that for sufficiently high activation degrees, all PNN architectures are filling. This conjecture was later proven independently in each of \cite{finkel2025activationdegreethresholdsexpressiveness, massarenti2025alexanderhirschowitztheoremneurovarieties, craciun2025linearindependencepowerspolynomials}. 
Furthermore, questions about non-defectiveness and identifiability of PNNs were settled separately by
Usevich et al. \cite{usevich2025identifiabilitydeeppolynomialneural} using techniques from tensor decompositions and by Massarenti and Mella \cite{massarenti2025alexanderhirschowitztheoremneurovarieties} by analyzing secant varieties.

Turning to convolutional architectures, the neuromanifold of a linear convolutional network was first studied in \cite{kohn2022geometrylinearconvolutionalnetworks} and \cite{kohn2024functionspacecriticalpoints}. 
In these works, convolution filters were identified with homogeneous bivariate polynomials and the convolution operation with polynomial multiplication. These allowed the authors to define the neuromanifold by factorization conditions and to show it is a semi-algebraic set; moreover, conditions under which it is filling are established.
Building on this line of work, Shahverdi \cite{shahverdi2024algebraiccomplexityneurovarietylinear} developed recursive algorithms to construct the polynomial equations and inequalities that define the neural manifold for a linear convolutional network, and it relates the number of critical points of the square loss function to the generic Euclidean degree of the Segre variety.

The work most closely related to this article is \cite{shahverdi2024geometryoptimizationpolynomialconvolutional}. In this work, Shahverdi et al. compute the dimension, the degree and the generic Euclidean degree for polynomial convolutional  networks in addition to classifying the general fiber and the singular locus. A central aspect of their analysis is based on the relation between  the neuromanifold and the Segre-Veronese variety; in particular, they show that the neuromanifold for CNN is birational to the Segre-Veronese variety.
Our work generalizes their dimension result for PGCNNs. 
In a later work \cite{shahverdi2026learning}, Shahverdi et al. considered neural networks and CNN for a generic polynomial as an activation functions, and they showed that for CNN the fibers are generally singletons and for fully connected neural networks fibers are finite.
Finally, in \cite{shahverdi2026identifiableequivariantnetworkslayerwise}, Shahverdi et al. treated identifiable equivariant models and showed that, under the condition of identifiability, these models are layerwise equivariant for some representation on latent spaces. 
An interesting feature of this result is that it is architecture-agnostic.

\paragraph{Notation.} We denote by $\K$a field with characteristic $0$ and its closure by $\overline{K}$.
For example, $\K$ can be the real numbers, the complex numbers, the rational numbers or p-adic numbers. 
For a field $\K$ and $n$ variables $x =(x_1, \dots, x_{n})$ we denote the ring of polynomials in $x$ over the field $\K$ by $\K[x_1, \dots, x_{n}]$ and the space of homogeneous polynomials in the variables $x$ of degree $d$ by $\Sym_\K(x, d)$. 
More generally, let the variables be partitioned into $m$ blocks $x^{(1)},\dots,x^{(m)}$, where $x^{(i)}=(x^{(i)}_1,\dots,x^{(i)}_{n_i})$. 
A polynomial $f \in \K[x^{(1)},\dots,x^{(m)}]$ is called \emph{multihomogeneous} of multidegree $(d_1,\dots,d_m)$ if it is homogeneous of degree $d_i$ with respect to each block $x^{(i)}$. 
We denote the space of such polynomials by $\Sym_\K\bigl((x^{(1)},\dots,x^{(m)}),(d_1,\dots,d_m)\bigr)$.

\section{Preliminaries on graded group algebras}\label{sec: preliminaries}
In this section, we present graded group algebras, their regular representations, and how Kronecker and Hadamard products act on them.
These are the building blocks for polynomial group convolutional networks. The section can be skipped, in case the reader is familiar with these concepts from algebra, and returned to later when needed. 

In the following, $G$ and $H$ are general finite groups with order $n=|G|$ and $|H|$ and $S:= \bigoplus_{i} S_i$ is a graded commutative ring such as the ring of polynomials in $n$ variables $\K[x_1, \dots, x_n]$.
For two elements $g, h\in G$ we denote their group product by $gh$ and for two elements $s, t \in S$ we denote their ring product by $st$ and their sum by $s+t$. 

\subsection{Graded group algebras and regular representations}
Here we present graded group algebras and their regular representations, which describe the $G$ action on group convolutional networks. 
We show, under a fixed choice of basis, these representations are described by $G$-circulant matrices, which are a generalization of the circulent matrices in the case where $G=\Z_n$.
\begin{definition}
 The graded group algebra $S[G]$ is an algebra
\[ 
S[G]:= \{\theta\ |\ \theta: G \to S\} \cong S^{|G|},
\]
where $S^{|G|}$ is the free $S$-module of rank ${|G|} $. The graded algebra bilinear product in $S[G]$ is given by \emph{the convolution operation,}
\[
(\theta \conv \psi)(g) := \sum_{h\in G}\theta(g h^{-1}) \psi(h),
\]
and the scalar multiplication is defined componentwise; i.e. for $s \in S$ we define $(s\theta)(g) := s\theta(g)$.
A basis of $S[G]$ is given by the group elements in $G$.
We call elements in $S[G]$ $S$-filters over $G$.
\end{definition}

\begin{remark}
    There are multiple equivalent ways to describe a filter $\theta \in S[G]$.
    In the definition it is described as a function from $G$ to $S$.
    Another useful way to describe this filter is to identify it with the formal linear combination 
    \[
    \theta \equiv \sum_{g\in G}\theta(g) g.
    \]
    Then the convolution defined above becomes just the product of such formal linear combinations.
    Note that if $G=(\Z, +)$, then the description of a finite filter as a formal linear combination identifies it with a polynomial in one variable, which is the identification of convolutional filters in linear CNN presented in \cite{kohn2022geometrylinearconvolutionalnetworks}.
    
    Also from here it is clear how $S[G]$ is isomorphic to the free $S$-module of rank ${|G|}$.
    Lastly, if we fix an order $(g_1, \dots, g_n)$ on the elements of $G$, then we can also identify the filter $\theta$ with the tuple $(\theta(g_1), \dots, \theta(g_n))$.
\end{remark}

\begin{example}
Let $G = \Z_3 = \{0,1,2\}$ be the cyclic group of order $3$ (with addition modulo $3$),
and let $S = \R$ concentrated in degree $0$.
Then
\[
\R[G] = \{ \theta : \Z_3 \to \R \} \cong \R^3.
\]

A filter $\theta \in \R[G]$ is determined by its values
\[
\theta(0)=a,\qquad \theta(1)=b,\qquad \theta(2)=c,
\]
and can be written as the formal linear combination
\[
\theta = a\,0 + b\,1 + c\,2.
\]
Similarly, let $\psi \in \R[G]$ be given by
\[
\psi(0)=x,\qquad \psi(1)=y,\qquad \psi(2)=z,
\quad\text{so that}\quad
\psi = x\,0 + y\,1 + z\,2.
\]
The convolution $\theta \conv \psi$ is the function $\Z_3 \to \R$ defined by
\[
(\theta \conv \psi)(g) = \sum_{h\in \Z_3} \theta(g-h)\psi(h),
\]
where subtraction is taken modulo $3$.
Explicitly,
\begin{align*}
(\theta \conv \psi)(0) &= a x + b z + c y,\\
(\theta \conv \psi)(1) &= a y + b x + c z,\\
(\theta \conv \psi)(2) &= a z + b y + c x.
\end{align*}
\end{example}

\begin{definition}[$S\lbrack G\rbrack$ regular representation]
The regular algebra representation $\rho$ of $S[G]$ on itself is defined as
\begin{align*}
\rho : S[G] &\to \End(S[G]), \qquad \rho_\theta(\psi) = \theta \conv \psi .\\
\theta &\mapsto \rho_\theta
\end{align*}
The algebra representation $\rho$ is completely determined by the induced group representation of $G$ on $S[G]$.
\end{definition}

\begin{definition}[$G$-circulent matrix]
Fix an order $(g_1,\dots,g_{n})$ of the elements of $G$. With respect to this basis, the map $\rho_\theta$ is represented by an $n \times n$ \emph{$G$-circulent matrix} $\Mat_\theta$ over $S$, whose entries are given by
\[
(\Mat_\theta)_{ij} := \theta(g_i g_j^{-1}).
\]

So the vector representation of $\Mat_\theta \psi$ under this base is given by
\[
(\Mat_\theta \psi)_i = \sum_{j = 1}^{n} (\Mat_\theta)_{ij} \psi(g_j),
\]
which is just a rewriting of the convolution $(\theta \conv \psi) (g_i)$.
In the special case $G = \mathbb{Z}_n$, the matrix $\Mat_\theta$ reduces to the classical circulent matrix.
\end{definition}

\begin{remark}\label{remark: Mat_is_rep}
    Note that since $\Mat$ is the algebra representation $\rho$ under a choice of base, it follows immediately that for any filters $\theta, \psi \in S[G]$ we get
    \[
    \Mat_\theta \Mat_\psi = \Mat_{\theta \conv \psi}\ \ \text{and}\ \ \Mat_\theta + \Mat_\psi = \Mat_{\theta + \psi}.
    \]
    Also, the set $\{\Mat_g\ |\ g\in G\}$ defines a basis for $G$-circulent matrices over the algebra $S$, where we identify $g\in G$ with the filter $g(h) = 1$ if $h=g$ otherwise $g(h)=0$.
\end{remark}

In the following proposition, we relate the invertibility of $G$-circulent matrices to the invertibility of filters as elements in the group algebra $S[G]$.

\begin{proposition}
    Let $\theta \in S[G]$. If $\Mat_\theta$ is invertible, then its inverse is also $G$-circulent.
    By Remark \ref{remark: Mat_is_rep}, $\Mat_\theta$ is invertible if and only if $\theta$ is invertible in $S[G]$.
\end{proposition}
\begin{proof}
   Since $S$ is commutative, we can apply Cayley-Hamilton to the inverse $\Mat_\theta^{-1}$ to write it as a polynomial in $\Mat_\theta$.
    Since $\Mat_\theta = \sum_{g\in G}\theta(g)\Mat_g$, then $\Mat_\theta^{-1}$ is also a linear combination of the matrices $\Mat_g$'s with coefficients in $S$. Therefore $\Mat_\theta^{-1}$ is also $G$-circulent.
\end{proof}

\subsection{Kronecker and Hadamard products}
Next we define the Kronecker product and the Hadamard product between filters and state propositions about their interactions with convolution. 
These two operations will give us two ways to define polynomial activation functions.

\begin{proposition}[Definition of Kronecker Product]
    There is a canonical isomorphism $\iota$ between $S[G] \otimes S[H]$ and $S[G\times H]$ given by
    \begin{align*}
        \iota: S[G] \otimes S[H] &\to S[G \times H]\\
        \theta \otimes \psi &\mapsto \iota(\theta \otimes \psi)
    \end{align*}
    where for $(g, h)\in G\times H$, $\iota(\theta \otimes \psi)(g, h) := \theta(g)\psi(h)$. By abuse of notation, we denote the filter $\iota(\theta \otimes \psi)$ by $\theta \otimes \psi$ and call it the Kronecker product of $\theta$ and $\psi$. 
\end{proposition}

\begin{remark}
    For $r\geq 1$, and $\theta\in S[G]$ we use the notation 
\[
\theta^{\otimes r} := \underbrace{\theta \otimes \hdots \otimes \theta}_r ,
\]
which is a filter in $S[G^r]$ where $G^r$ is the direct product of $r$ copies of $G$.
Note that the Kronecker product respects the grading; i.e. if $\theta\in S_k[G]$ and $\psi \in S_l[H]$ then $\theta \otimes \psi \in S_{k+l}[G\times H]$ and $\theta^{\otimes r} \in S_{rk}[G^r]$.
\end{remark}

\begin{definition}[Hadamard Product]
Let $\theta, \psi \in S[G]$, their \emph{Hadamard product} is the filter denoted by $\theta \circbd \psi \in S[G]$ and is given by
    \[
    (\theta \circbd \psi)(g) := \theta(g)\psi(g).
    \]
    Moreover, we denote the $r$th Hadamard power of $\theta$ by $\sigma_r(\theta):= \underbrace{\theta \circbd \hdots \circbd \theta}_{r}$. The Hadamard product also respects the grading of $S$.
\end{definition}

To justify our terminology, note that for any two filters $\theta$ and $\psi$ we have the identities
$
\Mat_{\theta\otimes \psi} = \Mat_{\theta} \otimes \Mat_{\psi}
$
and
$
\Mat_{\theta\circbd \psi} = \Mat_{\theta} \circbd \Mat_{\psi}
$
where on the right hand side in both of them, we use the Kronecker product and the Hadamard product between matrices, respectively.

\begin{example}
Let $G=H=\Z_2=\{0,1\}$ and $S=\R$.
Fix the ordered bases $\{0,1\}$ of $\R[G]$ and
$\{(0,0),(0,1),(1,0),(1,1)\}$ of $\R[G\times H]$.

Let $\theta\in \R[G]$ and $\psi\in \R[H]$ be given by
\[
\theta(0)=a,\ \theta(1)=b,
\qquad
\psi(0)=x,\ \psi(1)=y.
\]
A direct computation shows that left convolution by $\theta$ and $\psi$
is represented by the matrices
\[
\Mat_\theta=
\begin{pmatrix}
a & b\\
b & a
\end{pmatrix},
\qquad
\Mat_\psi=
\begin{pmatrix}
x & y\\
y & x
\end{pmatrix}.
\]

The Kronecker product of these matrices is
\[
\Mat_\theta\otimes \Mat_\psi=
\begin{pmatrix}
ax & ay & bx & by\\
ay & ax & by & bx\\
bx & by & ax & ay\\
by & bx & ay & ax
\end{pmatrix}.
\]

On the other hand, the filter $\theta\otimes\psi\in \R[G\times H]$ satisfies
$(\theta\otimes\psi)(g,h)=\theta(g)\psi(h)$, and convolution by
$\theta\otimes\psi$ on $\R[G\times H]$ is represented, in the chosen basis,
by the same matrix.
\end{example}

The following proposition shows that the Kronecker product distributes over convolution.
\begin{proposition}\label{prop: kprod_conv_kprod_=_conv_kprod_conv}
    Let $\theta_1, \theta_2 \in S[G]$  and $\psi_1, \psi_2 \in S[H]$. Then $(\theta_1 \kprod \psi_1) \conv (\theta_2 \kprod \psi_2) = (\theta_1 \conv \theta_2) \kprod (\psi_1 \conv \psi_2)$.
    By induction it holds that for $\theta_i, \psi_i \in S[G_i]$  for $1\leq i\leq r$
    \[
    (\theta_1 \kprod \hdots \kprod \theta_r) \conv (\psi_1 \kprod \hdots \kprod \psi_r) = (\theta_1 \conv \psi_1)\kprod \hdots \kprod (\theta_r \conv \psi_r).
    \]
\end{proposition}
\begin{proof}
Use definition of convolution, we get
    \begin{align*}
        (\theta_1 \kprod \psi_1) \conv (\theta_2 \kprod \psi_2)(g, h) &= \sum_{g', h'} (\theta_1 \kprod \psi_1)(g g'^{-1}, h h'^{-1}) (\theta_2 \kprod \psi_2)(g', h')\\
        \text{(use definition of $\kprod$)} &= \sum_{g', h'} \theta_1(g g'^{-1})\psi_1( h h'^{-1})\theta_2(g')\psi_2(h')\\
        \text{(reorder the terms)} &= \sum_{g', h'} \theta_1(g g'^{-1}) \theta_2(g')\psi_1( h h'^{-1})\psi_2(h')\\
        \text{(split the sum into product of two sums)} &= \left(\sum_{g'} \theta_1(g g'^{-1}) \theta_2(g')\right)\left( \sum_{h'}\psi_1( h h'^{-1})\psi_2(h')\right)\\
        &= (\theta_1 \conv \theta_2)(g) (\psi_1 \conv \psi_2)(h) = (\theta_1 \conv \theta_2) \kprod (\psi_1 \conv \psi_2) (g, h)
    \end{align*}
\end{proof}

On the other hand, the Hadamard product does not distribute over convolution, and this is easy to check as follows.
Let $\theta_1, \theta_2, \psi_1, \psi_2 \in S[G]$ then
\begin{align*}
        (\theta_1 \circbd \theta_2)\conv(\psi_1 \circbd \psi_2)&(g) = \sum_{h\in G}\theta_1(gh^{-1})\theta_2(gh^{-1})\psi_1(h)\psi_2(h)\\
        &\neq \left(\sum_{h\in G}\theta_1(gh^{-1})\psi_1(h)\right) \left(\sum_{h\in G}\theta_2(gh^{-1})\psi_2(h)\right) = (\theta_1 \conv \psi_1)\circbd (\theta_2 \conv \psi_2)(g).
\end{align*}

The following proposition explains the relation between the Hadamard product, the Kronecker product, and convolution.
This is an important proposition, as it tells us how to move between the two descriptions of polynomial activation functions; see §\ref{sec: pgcnn}.

\begin{proposition}\label{prop: circbd_kprod_conv_relation}
    Let $\theta_1, \psi_1, \theta_2, \psi_2 \in S[G]$. Then
    \[
    (\theta_1 \conv \psi_1) \circbd (\theta_2 \conv \psi_2) = \left((\theta_1 \otimes \theta_2) \conv (\psi_1 \otimes \psi_2)\right) |_G.
    \]
    It follows that for $\theta, \psi \in S[G]$ and $r\geq 1$ we get $\sigma_r(\theta \conv \psi) = (\theta^{\otimes r} \conv \psi^{\otimes r})|_G$, where $G$ is identified with the diagonal subgroup in $G^r$.
\end{proposition}
\begin{proof}
    By Proposition \ref{prop: kprod_conv_kprod_=_conv_kprod_conv},
    \[
    (\theta_1 \otimes \theta_2) \conv (\psi_1 \otimes \psi_2) = (\theta_1 \conv \psi_1) \otimes (\theta_2 \conv \psi_2).
    \]
    Restricting the right hand side to the diagonal subgroup $G < G^2$ gives us $(\theta_1 \conv \psi_1) \circbd (\theta_2 \conv \psi_2)$, as required.
\end{proof}

\begin{remark}
    If $H$ is a subgroup of $G$ then $S[H]$ is a subalgebra of $S[G]$. 
    Therefore, for $\psi \in S[H]$ and $\theta\in S[G]$ we have $\theta \conv \psi, \psi \conv \theta \in S[G]$. Namely,
    \[
    (\theta \conv \psi)(g) = \sum_{h\in H} \theta(gh^{-1}) \psi(h)\ \text{and}\   (\psi \conv \theta)(g) = \sum_{h\in H} \psi(h)\theta(h^{-1} g).
    \]
    In other words, any filter $\psi \in S[H]$ can be extended to $G$ by setting $\psi(g) = 0$ for all $g\notin H$.
    In  case we need to distinguish between $\psi \in S[H]$ and its extension $\psi \in S[G]$ we denote the extension by $\psi^G$ and its associated matrix by $\Mat_\psi^G$. 

\end{remark}

\subsection{General conditions and symbolic ranks}
To compute the dimension of the neuromanifold, which is the image of the parametrization map $\Phi$ (see §\ref{sec: pgcnn}), we will need to compute the rank of the Jacobian of the parametrization map $\Phi$ at a general set of parameters $\theta$.
In this subsection, we define what a general condition is, and we introduce the concept of symbolic matrix and symbolic ranks that will help us define general conditions.

\begin{definition}[generality]
    We say that a condition $\mathcal{P}$ is \emph{general} in $\K^m$ if it holds over a dense set in $\K^m$ equipped with the Zariski topology. 
\end{definition}

Another more intuitive way to understand generality, although not exactly equivalent, is to think of $\mathcal{P}$ as holding everywhere in  $\K^m$ except for a set of Lebesgue measure zero.

\begin{definition}
    A symbolic matrix $M$ is a matrix over a ring of polynomials $\K[y_1, \dots, y_m]$.
    We say $M$ is symbolic in terms of the indeterminates $(y_1, \dots, y_m)$.
    The symbolic rank of $M$ is decided by the vanishing of its minors in the ring $\K[y_1, \dots, y_m]$.
\end{definition}
For a symbolic matrix $M$ over $\K[y_1, \dots, y_m]$ and a point $p=(p_1, \dots, p_m) \in \K^m$, we construct the matrix $M(p)$ from $M$ by substituting every $y_i$ for $p_i$.
Then the condition on $p$ to satisfy $\rank(M) = \rank(M(p))$ is general in $\K^m$, because this condition fails if and only if $p$ is the zero of all minors of $M$ of size $\rank(M) \times \rank(M)$.
We illustrate this in the following example.

\begin{example}
Consider the symbolic matrix  over $\K[y_1,y_2]$
\[
M=
\begin{pmatrix}
y_1 & y_2 \\
y_2 & y_1
\end{pmatrix}
.
\]
The determinant of \(M\) is the polynomial
\[
\det M = y_1^2 - y_2^2.
\]
Hence, over the polynomial ring \(\K[y_1,y_2]\), the symbolic rank of \(M\) is \(2\)since the determinant is not the zero polynomial in $\K[y_1, y_2]$. Therefore, a point $p=(p_1, p_2) \in \K^2$ satisfies $\rank(M) = \rank(M(p_1, p_2))$ if and only if $p_1^2 - p_2^2 \neq 0$ which is a general condition in $\K^2$.
This example illustrates that the rank of a symbolic matrix is determined by polynomial conditions on the indeterminates and that rank deficiencies occur precisely on algebraic varieties defined by the vanishing of suitable minors.
\end{example}

\subsection{Invertibility of filters}
In this subsection, a general filter $\theta\in S[G]$ is shown to be invertible in the group algebra.
Similar results hold for its extension to $G^r$, its $r$th Hadamard product, and its $r$th Kronecker product. 
These results will be relevant in the proof of the main theorems. Since we also work in a general setting, we will be allowed to assume that the matrices we work with are full rank.

To comment on what is aimed for in these proofs, note that a general square matrix $M$ has, of course, a full rank because its determinant defines a nontrivial polynomial in terms of its entries. 
What we show in what follows is that for $G$-circulent matrices, which have weight-sharing constraints, the determinant still defines a nontrivial polynomial regardless of the group $G$.
\begin{definition}
    For a filter $\theta \in S[G]$, we define its determinant to be the determinant of its associated matrix, i.e. $\det(\theta) := \det(\Mat_\theta)$.
\end{definition}

In the following proposition, we relate the determinants of the three filters $\theta \in S[G]$ and $\theta^{\otimes}, \theta^{G^r} \in S[G^r]$.

\begin{proposition}\label{prop: det_formualae}
Let $\theta \in S[G]$, then we have 
\[
\det(\theta^{\otimes r})= \det(\theta)^{rn^{r-1}}\ \ \text{and}\ \ \det(\theta^{G^r}) = \det(\theta)^{n^{r-1}}
\]
where $n= |G|$. 
This implies that $\theta \in S[G]$ is invertible if and only if $\theta^{\otimes r}$ and $\theta^{G^r}$ in $S[G^r]$ are both invertible.
\end{proposition}
\begin{proof}
    The first equality follows directly from the property of the Kronecker product for matrices. 

    The second equality follows from a standard argument in representation theory. 
    Since we identify $G$ with the diagonal subgroup $\Delta < G^r$, we obtain that the restriction of the regular algebra representation of $S[G^r]$ to $S[G]$ has $|G^r/G| = n^{r-1}$ unique $S[G]$-invariant sub-representation $V_i \subset S[G^r]$ where $S[G^r] = \bigoplus_{i=1}^{n^{r-1}} V_i$ and  $S[G] \cong V_i$ for all $i$. Therefore, under a suitable choice of basis, we have $\Mat_\theta^{G^r}$ is a block diagonal matrix of $n^{r-1}$ blocks where each block is equal to $\Mat_\theta$.
    This proves our claim.
\end{proof}

In the following proposition, we state that invertibility of filters is a general condition in $S[G]$.

\begin{proposition}\label{prop:general conditions}
    Let $\theta \in S[G]$ then the following conditions are  general for all $r\geq 1$:\vspace{2mm}
    \begin{center}
        
    \begin{enumerate*}
        \item $\det(\theta) \neq 0$\hspace{4mm}
        \item $\det(\theta^{G^r}) \neq 0$\hspace{4mm} 
        \item $\det(\theta^{\otimes r}) \neq 0$\hspace{4mm}
        \item $\det(\sigma_r(\theta)) \neq 0$.
    \end{enumerate*}
    
    \end{center}
    \vspace{2mm}
\end{proposition}
\begin{proof}
    Assume $G = \{g_1, \hdots, g_{n}\}$ and $\theta=(\theta_{g_1}, \hdots, \theta_{g_{n}})$ where $g_1 = e$ is the group identity. 
    We have $(\Mat_\theta)_{i,j} = \theta(g_i g_j^{-1})$. For fixed $g_i$ (row) and  $g_k$ there is a unique $g_j$ such that $g_ig_j^{-1} = g_k$, namely $g_j=g_k^{-1} g_i$.
    Similarly for fixed  $g_j$ (column) and  $g_k$ there is a unique $g_i$ such that $g_ig_j^{-i} = g_k$.
    This means the element $\theta(g_k)$ occurs once and only once in each column and in each row in the matrix $\Mat_\theta$.
    Therefore by elementary rules of the determinant, the term $\theta(g_k)^n$ for all $g_k \in G$ occurs once and only once in the determinant of $\Mat_\theta$. 
    This implies the condition $\det(\Mat_\theta) = \det(\theta) \neq 0$ in $S[G]$ is general.
    By Proposition \ref{prop: det_formualae}, the conditions 2 and 3 are general.

    The same argument shows that $\det(\sigma_r(\theta))$  is not the trivial polynomial because it includes the terms $\theta(g_k)^{rn}$ for all $g_k \in G$.   
\end{proof}

\begin{corollary}\label{cor: general_filter_not_zero_divisors}
For $S = \K$ a field of characteristic zero,
Propositions \ref{prop: det_formualae} and \ref{prop:general conditions} imply that generally $\theta, \sigma_r(\theta) \in \K[G]$ and $\theta^{\otimes r}, \theta^{G^r} \in \K[G^r]$ are invertible, for all $r\geq 1$. In particular, they are generally not zero divisors in their algebras. 
\end{corollary}

\section{Polynomial group convolutional neural networks}\label{sec: pgcnn}

In this section, polynomial group convolutional networks and their neuromanifolds are defined, which are the main object of study here. 
Recall that the main goal of this article is to compute the dimension of these neuromanifolds.

\subsection{Definition of polynomial group convolutional neural networks}
\begin{definition}[GCNN and PGCNN]\label{def: GCNN}
    A group convolutional neural network (GCNN) with the filters $\theta:=(\theta_1, \dots, \theta_{L}) \in (\K[G])^L$ and an activation function $\sigma: \K \to \K$ is the map
    \begin{align*}
       \Phi_\theta: \K[G] &\to \K[G] \\
        x &\mapsto \sigma(\hdots(\sigma(\sigma(x \conv \theta_1) \conv \theta_2)\conv \hdots \conv \theta_{L-1}) \conv \theta_{L},
    \end{align*}
    where $\sigma$ acts component wise.
    When $\sigma = \sigma_r$, where $\sigma_r$ denotes the $r$th Hadamard product,  we call $\Phi_\theta$ a polynomial group convolutional neural network (PGCNN). 
    We call $x$ the input filter and $\theta_l$ the $l$th weight filter for $1\leq l \leq L$.
\end{definition}

The GCNNs are $G$-equivariant, where the $G$-action is the one given by the regular representation in Section §\ref{sec: preliminaries}, and equivariance follows directly from the associativity of convolution. Before discussing PGCNN further, we present a few remarks to connect our definition with existing work in the literature.

\begin{remark}
Convolutional neural networks, regardless of their name, are in the literature defined using the operation \emph{cross correlation} $\ccr$ instead of convolution $\conv$. Nevertheless, these two operations are intrinsically connected and interchangeable via the relation $\theta \ccr \psi := \theta \conv \overline{\psi}$, where the filter $\overline{\psi}(g) := \psi(g^{-1})$ is called the involution of $\psi$.
Therefore these two conventions are equivalent in the sense that every neural network defined via convolution can be defined via cross correlation by involuting all its filters. 
In this article, we adopt the convolutional convention for two reasons. Firstly, under this convention, the $G$-action translates to just one extra left convolution, and the equivariance property translates into associativity of the convolution operation. Secondly, by associativity of convolutions, we can study the interaction between the weight filters $\theta$ away from the input filter $x$.
\end{remark}

\begin{remark}
    In machine learning literature, the input filter $x$ is called the input signal, and the weight filters correspond to convolutional layers. 
    Moreover, the signal $x$ is considered over a domain $\Omega$ equipped with a $G$ action. In other words, $x: \Omega \to \K$ where for $t\in \Omega$ and $g\in G$ we have $(g\cdot x)(t) := x(g^{-1}\cdot t)$. 
    In this context, what we define above is a special case of GCNN where the domain $\Omega = G$.
    This simplification is convenient because it keeps the group algebra structure at the first layer.
\end{remark}

\begin{remark}
Notice that when $G=\Z_n$ is the $n$th cyclic group, the GCNN does \emph{not} specialize to CNN considered, for example, in \cite{kohn2022geometrylinearconvolutionalnetworks} and \cite{shahverdi2024geometryoptimizationpolynomialconvolutional} because these CNNs are  \emph{not equivariant} with respect to shifts.
The main reason is that in this CNN architecture, the domain is considered to be a sequence rather than a circle.
To see this, let $C_k$ denotes the domain that is a sequence of length $k$. 
Moreover, let $x$ be a signal on $C_n$ and $\theta$ be a signal (or filter) over $C_k$ where $k< n$.
Then their cross correlation with a step of length 1 will be the signal $x \ccr_{CNN} \theta$ over $C_{n-k+1}$ given by
\[
(x \ccr_{CNN} \theta)(l) := \sum_{j=0}^{k-1} x(l+j) \theta(j) , \, \text{where} \, l < n-k+1.
\]
Assume that $s_d$ denotes the shift to the right with $d$ steps; then 
$(s_d \cdot x) (l) := \theta(l-d \mod n)$, note importantly that the order of the modulo operation depends on the size of the signal $x$.
Given this, we have
\[
s_d \cdot(x \ccr_{CNN} \theta)(l) = (x \ccr_{CNN} \theta)(l-d \mod (n-k+1)) = \sum_{j=0}^{k-1}  x( [l-d \mod (n-k+1)] + j) \theta(j) ,
\]
and
\[
(s_d \cdot x \ccr_{CNN}  \theta)(l) = \sum_{j=0}^{k-1} s_d \cdot x(j + l) \theta(j)  = \sum_{j=0}^{k-1} x([l - d \mod n] + j) \theta(j) .
\]

This implies $s_d \cdot(x \ccr_{CNN} \theta) \neq s_d \cdot x \ccr_{CNN}  \theta$, unless $k=1$ which means the convolution filter contains only one element.
Hence, standard discrete CNNs with finite filters are not equivariant with respect to cyclic shifts, except in the trivial case $k=1$.
\end{remark}

\subsection{Neuromanifolds of PGCNNs}
From Definition \ref{def: GCNN}, we  see that for any tuple of $L$ filters $\theta:=(\theta_1, \dots, \theta_{L}) \in (\K[G])^L$ and  activation function $\sigma_r$, the  PGCNN $\Phi_\theta \in \Sym_\K(x, r^{L-1})[G]$ defines $n$ homogeneous polynomials in $n$ variables $x:=(x_1, \dots, x_{n})$ with degree $r^{L-1}$. Moreover, each coefficient in the homogeneous polynomial $\Phi_\theta(g)$ is itself a multihomogeneous polynomial with the multidegree $\mathbf{r} := (r^{L-1}, r^{L-2}, \dots, 1)$ in the parameters of the weight filters $\theta = (\theta_1, \dots, \theta_{L})$. Thus, these coefficients belong to the space $\Sym_\K(\theta, \mathbf{r})$. 

\begin{definition}[Neuromanifold]
    The \emph{neuromanifold} $\mathcal{M}_\Phi$ of the architecture of PGCNN of depth $L$ and activation degree $r$ is the image of the parameterization map $\Phi: \theta \mapsto \Phi_\theta$ in $\Sym_\K(x, r^{L-1})$; namely
    \[
    \mathcal{M}_\Phi = \{\Phi_\theta\ |\ \theta \in (\K[G])^L\} \subset \Sym_\K(x, r^{L-1})[G].
    \]
    We define the associated \emph{neurovariety} $\mathcal{V}_\Phi$ to be the Zariski closure of $\mathcal{M}_\Phi$ in $\Sym_\K(x, r^{L-1})[G]$.
\end{definition}

Note that the projection map $\Phi_\theta \mapsto \Phi_\theta(e
)$ over the identity element is injective due to equivariance. Therefore we can consider the ambient space of the neuromanifold $\mathcal{M}_\Phi$ to be $\Sym_\K(x, r^{L-1})$ rather than $\Sym_\K(x, r^{L-1})[G]$.

As we have seen until now, the activation function is represented using the Hadamard product. 
Another equivalent way is to replace the Hadamard product with the Kronecker product using Proposition \ref{prop: circbd_kprod_conv_relation}.
So in terms of the Kronecker product, the filter $\Phi_\theta$ can be written as
\[
\Phi_\theta = (x^{\otimes r^{L-1}} \conv \varphi_\theta)|_G,
\]
where $\varphi_\theta = \theta_1^{\otimes r^{L-1}} \conv \theta_2^{\otimes r^{L-2}} \conv \dots \conv \theta_{L} \in \K[G^{r^{L-1}}]$.  We define the parametrization map 
\begin{align*}
    \varphi: \K[G]^L &\to \K[G^{r^{L-1}}]\\
    \theta \mapsto \varphi_\theta
\end{align*} 
and we denote its image by $\mathcal{M}_\varphi := \{\varphi_\theta\ |\ \theta \in \K[G]^L\} \subset \K[G^{r^{L-1}}]$.

Both of the maps $\Phi$ and $\varphi$ are defined by multihomogeneous polynomials in the weight parameters $\theta=(\theta_1, \dots, \theta_L)$. Therefore, when $\K=\R$ or $\C$ the Tarski-Seidenberg Theorem and Chevalley Theorem, respectively imply that both $\mathcal{M}_\Phi$ and $\mathcal{M}_\varphi$ are semi-algebraic sets in their respective ambient spaces. 
Moreover, they induce rational maps $\tilde \Phi$ and $\tilde\varphi$ on the projectivizations of their domains and codomains.
The relation between the rational maps $\tilde \varphi$ and $\tilde \Phi$ is shown in a commutative diagram, see Figure \ref{fig: varphi_Phi_Lambda}, where $\Lambda$ is a projective linear map (it consists of convolution with $x^{\otimes r^{L-1}}$ and the restriction to $e\in G$).

\begin{remark}[Linear GCNN]
    In the linear case, where $r=1$, we have 
    \[
    \Phi_\theta = x\conv \varphi_\theta = x \conv \theta_1 \conv \dots \conv \theta_L.
    \]
    Note then that the coefficients of the polynomial $\Phi_\theta(e) \in \Sym_\K(x, 1)$ determine completely the filter $\varphi_\theta \in \K[G]$, and vice versa. 
    Moreover, it is clear that $\varphi$ is surjective.
    Therefore in the linear case, we have $\mathcal{M}_\Phi \cong \mathcal{M}_\varphi = \K[G]$.
\end{remark}

\begin{figure}
 \centering
\begin{tikzcd}
  \PP(\K[G])^L
  \arrow[to=1-3, bend left=25, "\tilde\Phi"]
    \arrow[r, "\tilde\varphi"]
  &
  \PP\mathcal{M}_\varphi \subset \PP(\K[G^{r^{L-1}}])
    \arrow[r, "\Lambda "]
  &
  \PP\mathcal{M}_\Phi \subset \PP(\Sym_\K(x, r^{L-1})) \\
\end{tikzcd}
\caption{A commutative diagram that describes the relation between the projectivized parametrization maps $\tilde \Phi$ and $\tilde \varphi$ induced by the PGCNN architecture with $L$ layers and activation degree $r$ over the finite group $G$.}
\label{fig: varphi_Phi_Lambda}
\end{figure}

\section{Main results on neuromanifolds of PGCNN}\label{sec: main_results}

In this section, we state  our main results about the neuromanifolds of PGCNN: we compute the dimensions $\mathcal{M}_\varphi$ and $\mathcal{M}_\Phi$, in addition to give a description of the general fibers of $\varphi$ and $\Phi$.

\subsection{Dimension of $\mathcal{M}_{\varphi}$ and $\mathcal{M}_{\Phi}$}
We describe here broadly our technique to compute the dimension of the neuromanifold $\mathcal{M}_\Phi$ and the dimension of $\mathcal{M}_\varphi$.
\begin{enumerate}
\item We study the rank of the Jacobians of the parametrization maps $\Phi$ and $\varphi$ at general filters.
Note that for nonzero scalars $\lambda_l \in \K-\{0\}$, where $1\leq l \leq L-1$, and any tuple of filters $\theta =(\theta_1, \dots, \theta_L)\in \K[G]^L$, we have 
\[
\varphi_\theta = \varphi_\psi\ \text{and}\ \Phi_\theta = \Phi_\psi, \numberthis \label{eq: scaling_fibers}
\]
where $\psi = (\lambda_1 \theta_1, \lambda_1^{-r }\lambda_2 \theta_2, \dots, \lambda_{L-2}^{-r} \lambda_{L-1}\theta_{L-1}, \lambda_{L-1}^{-r}\theta_L)$. 
Therefore the rank of the Jacobians $J_\theta \varphi$ and $J_\theta \Phi$ is bounded up by $nL - L + 1 = L(n-1)+1$.

\item The vectors that lie in the kernels of $J_\theta \varphi$ and $J_\theta \Phi$ due to Eq \eqref{eq: scaling_fibers} are all killed by the projectivization of the domain. Therefore if $J_\theta \tilde \varphi$ and $J_\theta \tilde \Phi$ are full rank, then  the affine Jacobians $J_\theta \varphi$ and $J_\theta \Phi$ have the  maximal rank $n(L-1)+1$.

\item To compute the rank of the Jacobians at general filters $\theta$, we consider the Jacobians as  symbolic matrices in terms of the parameters $\theta$ and compute their symbolic rank.

\end{enumerate}

\begin{theorem}\label{thm: tilde_varphi_is_regular}
    For activation degree $r\geq 2$, the rational map $\tilde{\varphi}$ is \emph{generally regular,} i.e. its Jacobian is full rank at a general point in its domain. This implies that $\dim(\mathcal{M}_\varphi)=L(n-1)+1$, where $n=|G|$ and $L$ is the number of layers.
\end{theorem}
\begin{proof}

    For $\theta := (\theta_1,\hdots, \theta_n)$ be general in $\K^{nL}$.
    We have $T_{[\theta]} (\PP(\K^n))^L = \bigoplus_{i} \K^{n}/\K\theta_i$, and if $[\dot{\theta}]\in \ker J_{[\theta]} \tilde{\varphi}$ then $J_\theta \varphi(\dot{\theta}) \in \K \varphi(\theta)$.
    By Euler theorem for homogeneous functions, we get that $J_\theta \varphi(\theta_1, 0, \hdots, 0) = r^{L-1} \varphi(\theta)$, and therefore $\dot{\theta} + k(\theta_1, 0, \dots, 0) \in \ker J_\theta \varphi$, for some $k \in \K$, which by Proposition \ref{prop: Jvarphi_kernel_has-dim_L-1} is contained in $\bigoplus_i \K \theta_i$.
    Hence $[\dot{\theta}] = 0 \in T_{[\theta]} (\PP(\K^n))^L$ and $J_{[\theta]} \tilde{\varphi}$ is full rank.

\end{proof}

The proof of Theorem \ref{thm: tilde_varphi_is_regular} relies strongly on the next proposition.
\begin{proposition}\label{prop: Jvarphi_kernel_has-dim_L-1}
    Let $L$ be the number of layers and $r$ be the degree of activation.
    Let also $\theta = (\theta_1, \hdots, \theta_L)$ denote the tuple of filters that define $\varphi_\theta$.
    Then for a general $\theta$, the kernel of the Jacobian $J_\theta\varphi$ has a basis of $L-1$ elements given by 
    \[
    \left((\theta_1, -r\theta_2, 0, \hdots, 0), (0, \theta_2, -r\theta_3, 0, \hdots, 0), \hdots, (0, \hdots, 0, \theta_{L-1}, -r\theta_L)\right).
    \]
\end{proposition}
\begin{proof}
    We do a proof by induction over the number of layers $L$. 
    The hardest step in this inductive proof is to show that for a perturbation $\dot \theta \in \ker(J_\theta \varphi)$ at a general point $\theta \in \K[G]^L$ the last layer filter $\theta_L$ and its corresponding perturbation filter $\dot \theta_L$ are related by a scaling, i.e. $\dot \theta_L = \lambda \theta_L$ for $\lambda \in \K\backslash\{0\}$. 
    For more details, see Appendix \ref{appendix: varph_is_regular}.
\end{proof}

The next theorem states the same result for the dimension of the neuromanifold $\mathcal{M}_\Phi$.

\begin{theorem}\label{thm: tilde_Phi_is_regular}
For activation degree $r\geq 2$, the rational map \(\tilde\Phi\) is \emph{generally regular}; that is, its Jacobian has full rank at a general point of its domain.
This implies  $\dim(\mathcal{M}_\Phi)=L(n-1) + 1$.
\end{theorem}

\begin{proof}
The claim follows directly from Proposition~\ref{prop: JPhi_kernel_has_dim_L-1}, 
in the same manner that 
Proposition~\ref{prop: Jvarphi_kernel_has-dim_L-1}
implies the general regularity of \(\tilde\varphi\) in 
Theorem~\ref{thm: tilde_varphi_is_regular}.
\end{proof}

\begin{proposition}\label{prop: JPhi_kernel_has_dim_L-1}
    Let $L$ be the number of layers and $r$ the activation degree.
    Let $\theta = (\theta_1,\dots,\theta_L)$ denote the tuple of filters defining $\Phi_\theta$.
    Then for a general $\theta$, the kernel of the Jacobian $J_\theta\Phi$ has an $(L-1)$--dimensional basis given by
    \[
    \bigl(\,(\theta_1,-r\theta_2,0,\dots,0),\ (0,\theta_2,-r\theta_3,0,\dots,0),\ \dots,\ (0,\dots,0,\theta_{L-1},-r\theta_L)\,\bigr).
    \]
\end{proposition}
\begin{proof}
    See Appendix \ref{appendix: Phi_is_regular}.
\end{proof}

\begin{remark}\label{rmk:dim_computation}
    The package \texttt{PGCNNGeometry} allows us to perform sanity checks for the dimension results for small networks (groups with up to 13 elements and network depth up to 3). 
    However, the numerical computation becomes intractable for bigger groups and deeper network because the size of the Jacobian matrix, for both maps $\varphi$ and $\Phi$, is at least quadratic in terms of the size of the group and exponential in terms of the number of layers.
\end{remark}

\subsection{General fiber of $\varphi$}
The fiber of $\varphi$ at filters $\theta:= (\theta_1, \dots, \theta_L)$ is the set
\[
\varphi^{-1}(\varphi_\theta) := \{\psi \in \K[G]^L\ |\ \varphi_\psi = \varphi_\theta\}.
\]
In this subsection, the shape of the general fiber of $\varphi$ is described. However, before we give this main result, we need to borrow a lemma about symmetric tensor ranks from \cite{comon2008symmetrictensorssymmetrictensor}.

\begin{lemma}[Lemma 5.1 \cite{comon2008symmetrictensorssymmetrictensor}]\label{lemma: symmetric_rank_of_linearly_indep_n_vector_is_n}

Let $y_1, \dots, y_n \in \K^{d}$ be linearly independent. Then the symmetric tensor 
\[
A:=\sum_{i=1}^n y_i^{\otimes r}
\]
has a symmetric rank $\rank_S(A) = n$.    
\end{lemma}

\begin{theorem}\label{thm: general_fiber_of_varphi}
    For general filters $\theta:=(\theta_1, \dots, \theta_L)$, we have $\tilde\varphi_\theta =\tilde\varphi_\psi$ if and only if 
    \[
    \psi=(\theta_1 \conv g_1,\  g_1^{-1}\conv \theta_2 \conv g_2,\  \dots,\  g_{L-2}^{-1}\conv \theta_{L-1} \conv g_{L-1},\ g_{L-1}^{-1}\conv \theta_L),
    \]
    where $g_i \in G$ for $1\leq i\leq L-1$, up to rescaling each filter.
\end{theorem}
\begin{proof}
     We do a proof by induction over the number of layers $L$. The base case is trivial and the fiber is a singleton, up to rescaling the only filter.

    Assume the proposition holds for all numbers of layers less than or equal to $L-1$.
    Let $\theta'$ and $\psi'$ denote the first $L-1$ filters in $\theta$ and $\psi$, and consider the case for $L$ to get
    \[
    \varphi_\theta = \varphi_{\theta'}^{\otimes r} \conv \theta_L = \lambda \varphi_{\psi'}^{\otimes r} \conv \psi_L = \lambda \varphi_\psi,
    \]
    where $\lambda \in \K\backslash\{0\}$.
    Since $\theta$ is general, by Proposition \ref{prop:general conditions} we get that $\varphi_\theta \in \K[G^{r^{L-1}}]$ has an inverse filter $\varphi_{\theta}^{-1}$ and, moreover, $\psi_L \in \K[G^{r^{L-1}}]$ has an inverse given by $\psi_L^{-1} = \varphi_\theta^{-1}\conv \lambda \varphi_{\psi'}^{\otimes r}$.
    Convolute both sides above by $\psi_L^{-1}$ to get
    \[
    \varphi_{\theta'}^{\otimes r} \conv \theta_L\conv \psi_L^{-1} = \lambda \varphi_{\psi'}^{\otimes r}.
    \]
    Define $\mu := \theta_L \conv \psi_L^{-1}$ and substitute in the above equation to get
      \[
    \varphi_{\theta'}^{\otimes r} \conv \mu = \lambda \varphi_{\psi'}^{\otimes r}.
    \]
Write $\mu = \sum_{g\in G} \mu(g) g$ and use Proposition \ref{prop: kprod_conv_kprod_=_conv_kprod_conv} to distribute convolution over the Kronecker product to get
\begin{align*}
  \sum_{g\in G} \mu(g) \varphi_{\theta'}\conv g \kprod \dots \kprod  \varphi_{\theta'}\conv g =  \lambda \varphi_{\psi'} \kprod \dots \kprod \varphi_{\psi'}.
\end{align*}
Notice that for any $g\in G$, $\varphi_{\theta'} \conv g$ is a vector in $\K^{n^{r^{L-2}}}$, and more importantly, it is a column in the matrix $\Mat_{\varphi_{\theta'}}$.
For general filters $\theta_1, \dots, \theta_{L-1}$, Proposition \ref{prop:general conditions} shows the determinant $\det(\Mat_{\varphi_{\theta'}}) \neq 0$ and so the vectors $\{\varphi_{\theta'} \conv g\ |\ g\in G \}$ are linearly independent. 
Therefore the vectors $\{\mu(g)^{1/r}\varphi_{\theta'} \conv g\ |\ g\in G, \mu(g)^{1/r} \neq 0 \}$ are linearly independent in $\overline{\K}^{n^{r^{L-2}}}$. 
By Lemma \ref{lemma: symmetric_rank_of_linearly_indep_n_vector_is_n}, the above equation holds if and only if $\mu \in G$ up to a scaling, which implies $\theta_L = g\conv \psi_L$ up to a rescaling. 
The rest follows by induction.
\end{proof}

\subsection{General fiber of \(\Phi\)}
In this subsection, we discuss the general fiber of the parametrization map $\Phi$. 
To do that, we develop a Proposition similar in nature to Lemma \ref{lemma: symmetric_rank_of_linearly_indep_n_vector_is_n}, and by using it we are able to show that the general fiber $\Phi$ is the same as the general fiber of $\varphi$. 

Denote by $VP_{n, k, d}$ the projective variety of $k$th power of degree $d$ homogeneous polynomials in $n$ variables $x$ \cite{Abdesselam_2014}. Namely,
\[
VP_{n, k, d} := \{[p] \in \PP(\Sym_\K(x, kd))\ |\ \exists q \in \Sym_\K(x, d),\ q^k = p\}.
\]

\begin{proposition}\label{prop: intersection_of_general_secant_is_trivial}
    Let $G=\{g_1, \dots, g_n\}$ as a set, and let the activation degree $r\geq 2$.
    For general filters $\theta=(\theta_1, \dots, \theta_L)$, where $L\geq 1$, and $p_i(\theta) := \sigma_r(\Phi_\theta)(g_i)$ we have
    \[
   \PP( \langle p_1(\theta), \dots, p_n(\theta) \rangle )\cap\  VP_{n, r, r^{L-1}} = \{[p_1(\theta)], \dots, [p_n(\theta)]\},
    \numberthis, \label{eq: intersection_scheme_linear_span_with_VP}\]
    where $\langle p_1(\theta), \dots, p_n(\theta)\rangle$ denotes the linear span of these polynomials in $\Sym_\K(x, r^L)$.
\end{proposition}
\begin{proof}
    We consider the incidence scheme
    \[
    Y:= \{(\theta, [F]) \in \Theta^o \times VP_{n, r, r^{L-1}}\ |\ p_1(\theta) \wedge \dots \wedge p_n(\theta) \wedge F = 0 \in \bigwedge^{n+1} \Sym_\K(x, r^L) \},
    \]
    where $\Theta^o \subset \Theta$ is the open subspace of the parameters such that $\{ p_1(\theta), \dots, p_n(\theta)\}$ are linearly independent.
    We consider the projection map $\pi: Y \to \Theta^o$ and note that for $\theta\in \Theta^o$ the fiber $\pi^{-1}(\theta)$ is the intersection defined by \eqref{eq: intersection_scheme_linear_span_with_VP}.
    Using two different semi continuity theorems from algebraic geometry we restrict the dimension of fiber to be $0$ for a general element $\theta \in \Theta^o$ and then we restrict its size (length) to be equal to $n$.

    Note that the first part of the argument, where we restrict the dimension of the fiber, is similar to an argument in \cite[Theorem 22]{finkel2025activationdegreethresholdsexpressiveness} in which they compute the dimension of the neurovariety for a fully connected neural networks with layers of equal width. Our argument extends this to determine the exact size of the general fiber.
    See Appendix \ref{appendix: proof_proposition_4.8} for the complete proof.
\end{proof}
\begin{theorem}\label{thm: general_fiber_of_Phi}
    For general filters $\theta:=(\theta_1, \dots, \theta_L)$, we have $\tilde \Phi_\theta =\tilde \Phi_\psi$ if and only if $\psi:=(\theta_1 \conv g_1,\  g_1^{-1}\conv \theta_2 \conv g_2,\  \dots,\  g_{L-2}^{-1}\conv \theta_{L-1} \conv g_{L-1},\ g_{L-1}^{-1}\conv \theta_L)$, where $g_i \in G$ for $1\leq i\leq L-1$, up to rescaling each filter.
\end{theorem}

\begin{proof}
    The proof is similar to Theorem \ref{thm: general_fiber_of_varphi} and proceeds by induction over the number of layers $L$. The base case is trivial and the fiber is a singleton, up to rescaling.

    Assume the theorem holds for all numbers of layers less than or equal to $L-1$.
    Let $\theta'$ and $\psi'$ denote the first $L-1$ filters in $\theta$ and $\psi$, and consider the case for $L$ to get
    \[
    \Phi_\theta 
    = \sigma_r(\Phi_{\theta'}) \conv \theta_L 
    = \lambda \sigma_r(\Phi_{\psi'}) \conv \psi_L 
    = \lambda \Phi_\psi, \numberthis \label{eq: Phi_theta_equal_Phi_psi}
    \]
    where $\lambda \in \K-\{0\}$. 
    We show $\psi_L$ is invertible in $\K[G]$. 
    Let $S= K[x_g\ |\ g\in G]$ be the polynomial ring over $\K$  and $\operatorname{Frac}(S)$ be its field of fraction.
    Since $\theta$ is general by Proposition \ref{prop: det_formualae} $\det \Phi_\theta \neq 0$ in $S[G]$ and therefore $\Phi_\theta$ is invertible in $\operatorname{Frac}(S)[G]$. 
    Then Eq \eqref{eq: Phi_theta_equal_Phi_psi} implies $\psi_L$ has a left inverse in $\operatorname{Frac}(S)[G]$ which is a finite dimensional algebra and therefore $\psi_L$ has a right inverse, too, and it is invertible in $\operatorname{Frac}(S)[G]$. However, $\psi_L \in \K[G] \subset \operatorname{Frac}(S)[G]$. Therefore $\psi_L$ inverse in $\K[G]$ exists.
    Define $\mu:= \theta_L \conv \psi_L^{-1}$ and substitute in the last equation to get
        \[
    \sigma_r(\Phi_{\theta'}) \conv \mu = \lambda \sigma_r(\Phi_{\psi'}).
    \]
    Then evaluate at $e \in G$ to get
            \[
    \sum_{g\in G}\sigma_r(\Phi_{\theta'})(g) \mu(g^{-1}) = \lambda \sigma_r(\Phi_{\psi'})(e).
    \]
    The last equation says that $\lambda \sigma_r(\Phi_{\psi'})(e)$ lies in the linear span of $\{\sigma_r(\Phi_{\theta'})(g)\}_{g\in G}$. 
     By Lemma \ref{prop: intersection_of_general_secant_is_trivial}, we get $\lambda \sigma_r(\Phi_{\psi'})(e) = \sigma_r(\Phi_{\theta'})(g)$ for some $g\in G$. 
    As a result, we have $\mu = \lambda g$ for some $g\in G$, or equivalently $\theta_L = \lambda g\conv \psi_L$.
    The rest follows by induction.    
    
\end{proof}

According to the definition of identifiability in \cite{usevich2025identifiabilitydeeppolynomialneural}, Theorem \ref{thm: general_fiber_of_Phi} states that a PGCNN is \emph{globally identifiable} for all depth $L$ and all activation degree $r\geq 2$.

\begin{remark}
    In \cite{finkel2025activationdegreethresholdsexpressiveness}, it is shown that
    the powers of non-proportional multivariate polynomials are linearly independent for all powers $r>\tilde{r}(k)$ where $\tilde{r}$ is a function that depends only the number of polynomials $k$. This is then used to find an activation degree threshold for an arbitrary architecture of fully connected networks. A similar line of reasoning can be followed to prove weaker version Proposition \ref{prop: intersection_of_general_secant_is_trivial}, where $r$ is taken to be greater than $\tilde{r}(|G|)$. 
    However, in order to prove Proposition \ref{prop: intersection_of_general_secant_is_trivial} for $r\geq 2$, we need to resort to the use of semi continuity theorems of fiber dimension and fiber length from algebraic geometry.
\end{remark}

\begin{remark}\label{rmk:fiber_computation}
We have verified, using \texttt{PGCNNGeometry}, that Theorem \ref{thm: general_fiber_of_Phi} holds for the PGCNNs described in Table \ref{tab:Phi_general_fiber_checks}.  
This check is limited to small groups and networks. \texttt{PGCNNGeometry}  computes the size of  a fiber $F_p$ of $\tilde \Phi$ over a general point  $p$ in the image of $\tilde \Phi$ via an algorithm that relies on computing the degree of the vanishing ideal $I$ over $F_p$ in the multi-projective space $(\PP(\K^n))^L$.
Since the ideal $I$ is over a multi-projective space, we need to saturate with respect to the the irrelevant ideal in order to exclude the irrelevant zero locus from the fiber. This saturation process is very expensive for larger groups and/or deeper networks.
Recall that a tuple of filter $\theta$ belongs to the irrelevant zero locus if and only if at least one of the filters is the zero filter.
\begin{table}[h]
    \centering
    \begin{tabular}{c|c |c|c|c |c|c|c |c}
        Group &  $\mathbb{Z}_2$ & 
         $\mathbb{Z}_3$ & 
         $\mathbb{Z}_4$ & 
         $\mathbb{Z}_2 \times \mathbb{Z}_2$ & 
         $\mathbb{Z}_5$ &
         $\mathbb{Z}_6$ &
         $\mathbb{Z}_2 \times \mathbb{Z}_3$ & 
         $S_3$\\
         \hline
         Max no. of layers & 4 & 3 & 2 &
         2 &2 &2 &2 &2 
         
    \end{tabular}
    \caption{Verification of Theorem \ref{thm: general_fiber_of_Phi}, using \texttt{PGCNNGeometry}, for PCGNN with the listed groups and maximum depth as shown in the second row; the activation degree is fixed to 2.} 
    
    \label{tab:Phi_general_fiber_checks}
\end{table}
\end{remark}

\section{Conclusion and Outlook}
In this work, we presented a mathematical framework for polynomial group convolutional networks using graded group algebras.
Under this framework, we computed the dimension of the neuromanifold for a PGCNN where the group $G$ is an arbitrary finite group.
We showed that the dimension does not depend on the group structure but only on the size of the group. 
Furthermore, for an activation degree $r\geq 2$, we study the general fibers of the two parametrizations $\varphi$ and $\Phi$ and show that in a general fiber the corresponding layers differ only by a rescaling and a convolution with a group element. In total this shows that a PGCNN is globally identifiable for $r\geq 2$.

For future work, we aim to compute the degree and the generic Euclidean degree for the neuromanifold of PGCNN, we aim to generalize the techniques used in \cite{shahverdi2024geometryoptimizationpolynomialconvolutional} for the same computation for CNN.
In the CNN case, considered in \cite{shahverdi2024geometryoptimizationpolynomialconvolutional}, the parametrization map $\Phi$ factors through Segre-Veronese embedding via a projective linear map.
Since the degree of a variety is invariant under projective linear morphism, the authors get that the degree of the Segre-Veronese variety naturally translates into the degree of the neurovariety of a CNN.
In contrast, in the case of PGCNN, the parametrization map $\Phi$ factors through the Segre-Veronese embedding via a rational map with a nontrivial kernel rather than a projective linear morphism.
Therefore, in the case of PGCNN, we expect the Segre-Veronese factorization of the map $\Phi$ to give an upper bound for both the degree and the generic Euclidean degree, rather than determining them exactly.

Another important finding of neuroalgebraic geometry  is that points in the singular locus of the neuromanifold have a higher probability of being extrema under training \cite{marchetti2025algebraunveilsdeeplearning}.
This makes the singular locus of the neuromanifold of PGCNN worthy of further investigation.
During our experimentation we observed that the rank of the Jacobian, for both maps $\varphi$ and $\Phi$, drops when the matrix representation of any convolution layer is not full rank.
This suggests that there is a connection between singularities of layer, locally, and singularities of the network, globally.

For further generalization, we propose two natural directions. The first direction is to consider polynomial activation functions which are sums of Hadamard products. 
This extension is interesting since, as we discussed before, PNNs are being used in multiple fields now.
Moreover, by Weierstrass Approximation Theorem, these networks can be used to approximate neural networks with non-polynomial activation functions.
This might enable a deeper investigation of the geometric structure underlying GCNNs with non-polynomial activations.

Another direction is to relax our restriction to finite groups, and study PGCNNs of compact groups.
In this case input signals and convolution filters are local sections of an associated vector bundle over a homogeneous space, and the convolution operation is formulated as an integral \cite{10.5555/3454287.3455107}. 
To consider this case, we either need to generalize our framework to use infinite dimensional graded group algebras, or to discretize compact groups to finite groups and reduce to the current case.
While in this setting one would have an infinite-dimensional generalization of a neuromanifold, relevant questions still remain about fibers and singularities.

\section*{Acknowledgements}
We thank  Kathl\'en Kohn,  Giorgio Ottaviani, Vahid Shahverdi and Konstantin Usevich  for discussions and suggestions. 

This work was in part supported by the Wallenberg AI, Autonomous Systems and Software Program (WASP)
funded by the Knut and Alice Wallenberg Foundation.
The computations were partly
enabled by resources provided by the National Academic Infrastructure for Supercomputing in Sweden
(NAISS), partially funded by Vetenskapsrådet through grant agreement no. 2022-06725.

\bibliographystyle{alpha}
\bibliography{ref}

\appendix

\section{Proof of Proposition \ref{prop: Jvarphi_kernel_has-dim_L-1}}\label{appendix: varph_is_regular}

Before we state the proof of Proposition \ref{prop: Jvarphi_kernel_has-dim_L-1}, we need to set some notation that applies only in this appendix. We also need to prove two lemmata, one easy and one hard.

\paragraph{Notation.}
\begin{itemize}
    \item We introduce the following notation
    \[
    [\breve{x}_1 \diamond x_2 \diamond\hdots\diamond x_n] := \breve{x}_1 \diamond x_2\diamond \hdots\diamond x_n 
    \, +  \,
    x_1 \diamond\breve{x}_2 \diamond\hdots\diamond x_n \, + \hdots \, + \,
    x_1 \diamond\hdots\diamond \breve{x}_n,
    \]
    where $\breve{}$ and $\diamond$ denote arbitrary operations.
    \item  Let $g \in G^{r}$ and $h\in G$ we define $g\cdot h := g \cdot h^r$.
\end{itemize}

\begin{lemma}\label{lemma: jacobian_rules_conv_kprod}
    Let $\theta$ and $\psi$ be $G$-filters, then
    \begin{enumerate}
        \item   
        \[
    J_{\theta , \psi}(\theta \conv \psi)(\dot{\theta}, \dot{\psi}) = \dot{\theta}\conv \psi + \theta\conv \dot{\psi} = [\dot{\theta}\conv \psi].
    \]
    \item
\begin{align*}
        J_{\theta}(\theta^{\otimes r})(\dot{\theta}) &= \dot{\theta}\kprod \theta \kprod\hdots\kprod \theta + \theta \kprod \dot{\theta}  \kprod\hdots\kprod \theta + \hdots + \theta \kprod \hdots \kprod \theta \kprod \dot{\theta}\\
        &= \dot{\theta} \kprod \theta^{\otimes (r-1)} + \theta \kprod \dot{\theta} \kprod \theta^{\otimes (r-2)} + \hdots + \theta^{\otimes (r-1)} \kprod \dot{\theta}\\
        &= [\dot{\theta}\kprod \theta \kprod\hdots\kprod \theta].
\end{align*}

    \item Let $\theta = (\theta_1, \hdots, \theta_L)$, $\dot{\theta} = (\dot{\theta}_1, \hdots, \dot{\theta}_L)$. Moreover, let $\theta'$ and $\dot{\theta}'$ be the tuples containing the first $L-1$ filters in $\theta$ and $\dot{\theta}$, respectively. 
    Since $\varphi_\theta = \varphi_{\theta'}^{\otimes r} \conv \theta_L|^{G^{r^{L-1}}}$, we get
    \[
    J_\theta\varphi(\dot{\theta}) = [J_{\theta'}\varphi(\dot{\theta}') \kprod \underbrace{\varphi_{\theta'} \kprod \hdots \kprod \varphi_{\theta'}}_{= \varphi_{\theta'}^{\otimes (r-1)}}] \conv \theta_L + \varphi_{\theta'}^{\otimes r} \conv \dot{\theta}_L.
    \]
    \end{enumerate}
\end{lemma}
\begin{proof}
    \begin{enumerate}
        \item 
        \begin{align*}
            J_{\theta , \psi}(\theta \conv \psi)(\dot{\theta}, \dot{\psi}) &= J_{\theta , \psi}\left(\left(\sum_h \theta(g_1 \cdot h^{-1})\psi(h), \hdots, \sum_h \theta(g_n \cdot h^{-1})\psi(h)\right)\right)(\dot{\theta}, \dot{\psi}).
        \end{align*}
        Since $ J_{\theta , \psi}\left(\sum_h \theta(g_i \cdot h^{-1})\psi(h)\right)(\dot{\theta}, \dot{\psi}) = \sum_h \dot{\theta}(g_i \cdot h^{-1})\psi(h) + \sum_h \theta(g_i \cdot h^{-1})\dot{\psi}(h)$, we get
    \[
    J_{\theta , \psi}(\theta \conv \psi)(\dot{\theta}, \dot{\psi}) = \dot{\theta}\conv \psi + \theta\conv \dot{\psi}.
    \]

    \item We prove this for $r=2$, the general case follows by induction.

    We have $\theta \kprod \theta = (\theta_1^2, \theta_1\theta_2, \hdots , \theta_1\theta_n, \theta_2\theta_1, \hdots, \theta_2\theta_n, \hdots, \theta_n \theta_1, \hdots, \theta_n^2)$, hence $J_\theta(\theta \kprod \theta)$ is a $n^2 \times n$ matrix where the $((i-1)n + j)$th row corresponding to the function $\theta_i\theta_j$ will have $\theta_j$ at the $i$th column, $\theta_i$ a the $j$th column, and zero everywhere else.
    Therefore, the $((i-1)n + j)$th row in $J_\theta(\theta \kprod \theta)(\dot{\theta})$ is equal to $\theta_i\dot{\theta}_j + \dot{\theta}_i\theta_j$.

    \item Direct application of the first rule, then the second rule.
    \end{enumerate}
\end{proof}

\begin{lemma}\label{lemma: jac_varphi_simplified_case}
    Let $\theta = (\theta_1, \hdots, \theta_{L-1}, e) \in \K[G]^L$ where $\theta_i$ is a general filter for $1\leq i < L$ and $e \in G \subset \K[G]$ is the identity element. If $\dot \theta = (\dot \theta_1, \dots, \dot \theta_L) \in \ker J_\theta \varphi$ then $\dot \theta_L = \lambda e$ for some $\lambda \in \K$.
\end{lemma}
\begin{proof}
Write $\theta'=(\theta_1,\dots,\theta_{L-1})$ and similarly $\dot\theta'=(\dot\theta_1,\dots,\dot\theta_{L-1})$.
For brevity, write,
\[
J := J_{\theta'}\varphi(\dot{\theta}').
\]
By Lemma \ref{lemma: jacobian_rules_conv_kprod} (the Jacobian computation rules),
$\dot\theta\in\ker\bigl(J_{\theta}\varphi\bigr)$ if and only if

\begin{equation}\label{eq:jac_kernel_at_theta_L_equal_e}
\bigl[J\kprod \varphi_{\theta'}^{\otimes (r-1)}\bigr]
\;=\;
-\,\varphi_{\theta'}^{\otimes r}\conv \dot\theta_L
\end{equation}
Evaluate \eqref{eq:jac_kernel_at_theta_L_equal_e} at tuples of the special form $\tilde g_i=(g_i,\dots,g_i)\in G^{r^{L-1}}$
(where $g_i\in G^{r^{L-2}}$). This yields, for each $i$,
\[
r\,(\varphi_{\theta'}(g_i))^{r-1} J(g_i) \;=\; \sum_{h\in G} (\varphi_{\theta'}(g_i h^{-1}))^r \dot\theta_L(h).
\]
Hence
\begin{equation}\label{eq:J_at_gi}
J(g_i)
\;=\;
\frac{1}{r}\,\varphi_{\theta'}(g_i)\,\dot\theta_L(e)
\;+\;
\frac{1}{r}\sum_{h\in G\setminus\{e\}}
\frac{(\varphi_{\theta'}(g_i   h^{-1}))^r}{(\varphi_{\theta'}(g_i))^{\,{r-1}}}\,\dot\theta_L(h).
\end{equation}

Now evaluate \eqref{eq:jac_kernel_at_theta_L_equal_e} at $\tilde g=(g_{i_1},\dots,g_{i_r})$ to get,
\[
\sum_{k=1}^r J(g_{i_k}) \prod_{j\neq k} \varphi_{\theta'}(g_{i_j}) = \sum_{h\in G}\varphi_{\theta'}^{\otimes r}(\tilde{g}   h^{-1})\dot{\theta}_L(h).
\]
Substituting \eqref{eq:J_at_gi} in the above equation gives a linear system in the unknowns $\{\dot\theta_L(h)\}_{h\neq e}$ with coefficient matrix $C$ whose entries are symbolic expressions in the values $\varphi_{\theta'}(g_i)$. Denote these coefficients by 
\[
c_{\tilde{g}, h} = \dfrac{1}{r} \sum_{k=1}^r \left(\dfrac{(\varphi_{\theta'}(g_{i_k}   h^{-1}))^r\prod_{j\neq k}\varphi_{\theta'}(g_{i_j})}{(\varphi_{\theta'}(g_{i_k}))^{r-1}}\right) - \varphi_{\theta'}^{\otimes r}(\tilde{g} h^{-1}).
\]

Observe $c_{\tilde g,e}=0$ for every $\tilde g$ that is why we exclude the variable $\dot{\theta}_L(e)$ from the unknowns. The matrix $C$ has $|G^{r^{L-1}}|=n^{r^{L-1}}$ rows and $n-1$ columns, so $C$ is full column rank if and only if some $(n-1)\times(n-1)$ minor is nonzero as a symbolic polynomial in the variables $\theta'$.

Choose the $(n-1)$ rows corresponding to
\[
\tilde g^{\,i}:=(g_1, g_i,\dots,g_i)\qquad(2\le i\le n),
\]
where $g_j=(h_j,\dots,h_j)\in G^{r^{L-2}}$, $h_1=e$ and $h_j \in G$. For this choice, one checks that the diagonal entries of the resulting $(n-1)$-minor contain a factor $(\varphi_{\theta'}(g_1))^{r+1}$ which does not appear elsewhere in that minor. 
Namely,
 \[
 c_{\tilde{g}^i, h_j} = \dfrac{1}{r}\dfrac{(\varphi_{\theta'}(g_i))^{r-1}}{(\varphi_{\theta'}(g_1))^{r-1}} \varphi_{\theta'}((g_1    h_j^{-1}))^r 
 + 
 \dfrac{r-1}{r}\dfrac{(\varphi_{\theta'}(g_i   h^{-j}))^r}{\varphi_{\theta'}(g_i)} \varphi_{\theta'}(g_1) 
 - 
 \varphi_{\theta'}^{\otimes r}(\tilde{g}^{i}    h_j^{-1}).
 \]
and 
 \[
 c_{\tilde{g}^i, h_i} = \dfrac{1}{r}\dfrac{(\varphi_{\theta'}(g_i))^{r-1}}{(\varphi_{\theta'}(g_1))^{r-1}} \varphi_{\theta'}((g_1    h_i^{-1}))^r 
 + 
 \dfrac{r-1}{r}\dfrac{(\varphi_{\theta'}(g_1))^{r+1}}{\varphi_{\theta'}(g_i)}
 - 
 \varphi_{\theta'}^{\otimes r}(\tilde{g}^{i}    h_i^{-1}).
 \]

Since $\varphi_{\theta'}(g_1)$ has monomials in the parameters of $\theta'$ that do not occur  in  $\varphi_{\theta'}(g_i)$ for $2 \geq i \leq n$, then this symbolic minor is nonvanishing for general $\theta'$, so $C$ is symbolically full rank. Therefore the only solution is $\dot\theta_L(h)=0$ for all $h\ne e$. 
\end{proof}

We are finally ready to give a proof of Proposition \ref{prop: Jvarphi_kernel_has-dim_L-1}.
\begin{proof}[Proof of Proposition \ref{prop: Jvarphi_kernel_has-dim_L-1}]

Let the perturbations be $\dot\theta=(\dot\theta_1,\dots,\dot\theta_L)$ and write $\theta'=(\theta_1,\dots,\theta_{L-1})$ and similarly $\dot\theta'=(\dot\theta_1,\dots,\dot\theta_{L-1})$.
We prove the claim by induction on $L$.

\medskip\noindent\textbf{Base case $L=1$.}
If $L=1$ then $\varphi_\theta=\theta_1$ and
\[
J_{\theta}\varphi(\dot\theta)=J_{\theta_1}(\theta_1)(\dot\theta_1)=\dot\theta_1,
\]
so $J_{\theta}\varphi$ is the identity and hence full rank. The kernel is trivial, which matches the statement for $L=1$.

\medskip\noindent\textbf{Inductive step.}
Assume the proposition holds for $L-1$ layers; we prove it for $L$ layers.
For brevity, write,
\[
J := J_{\theta'}\varphi(\dot{\theta}').
\]
By Lemma \ref{lemma: jacobian_rules_conv_kprod} (the Jacobian composition / convolution rule),
$\dot\theta\in\ker\bigl(J_{\theta}\varphi\bigr)$ if and only if
\begin{equation}\label{eq:jac_kernel}
\bigl[J\kprod \varphi_{\theta'}^{\otimes (r-1)}\bigr]\conv \theta_L
\;=\;
-\,\varphi_{\theta'}^{\otimes r}\conv \dot\theta_L.
\end{equation}

If $\dot \theta_L = 0$ then $J = 0$ the result follows by the inductive step.
Assume $\dot \theta_L \neq 0$, convolute both sides of \eqref{eq:jac_kernel} with the inverse of $\theta_L$, which exists because $\theta_L$ is general, and let $\dot{\tilde{\theta}}_L := \dot \theta_L \conv \theta_L^{-1}$ to get 

\begin{equation}\label{eq:jac_kernel_mult_by_theta_L_inverse}
\bigl[J\kprod \varphi_{\theta'}^{\otimes (r-1)}\bigr]
\;=\;
-\,\varphi_{\theta'}^{\otimes r}\conv \dot{\tilde{\theta}}_L.
\end{equation}

Since $\theta'$ is general and $\dot \theta_L \neq 0$, then by Lemma \ref{lemma: jac_varphi_simplified_case}, we have $\dot{\tilde{\theta}}_L = \lambda e$ for some nonzero $\lambda \in \K$.
This implies $\dot \theta_L = \lambda \theta_L$, and
Therefore, any $\dot\theta\in\ker J_\theta\varphi$ must have the form
\[
\dot\theta = (\dot\theta_1,\dots,\dot\theta_{L-1},\;\lambda\theta_L).
\]
Observe that the vector
\[
(0,\dots,0,\;-\tfrac{\lambda}{r}\theta_{L-1},\;\lambda\theta_L)
\]
lies in the kernel as well (it is a direct check using the Jacobian composition rule). Subtracting this vector from $\dot\theta$ yields
\[
(\dot\theta_1,\dots,\dot\theta_{L-1}-(-\tfrac{\lambda}{r}\theta_{L-1}),\;0)
\;=\;
(\dot\theta_1,\dots,\dot\theta_{L-1}+\tfrac{\lambda}{r}\theta_{L-1},\;0)
\in \ker J_{\theta}\varphi.
\]
By Eq.~\eqref{eq:jac_kernel}, the vector 
\[(\dot\theta_1,\dots,\dot\theta_{L-1}+\tfrac{\lambda}{r}\theta_{L-1}) \in \ker J_{\theta'}\varphi,
\]
and by the inductive hypothesis, it can be written as a linear combination of
\[
\bigl((\theta_1,-r\theta_2,0,\dots,0),\,
(0,\theta_2,-r\theta_3,0,\dots,0),\,
\dots,\,
(0,\dots,0,\theta_{L-2},-r\theta_{L-1})\bigr).
\]
Adding back the previously subtracted vector
$(0,\dots,0,-\tfrac{\lambda}{r}\theta_{L-1},\lambda\theta_L)$
shows that $\dot\theta$ is a linear combination of the $L-1$ vectors stated in the proposition.
This completes the induction and the proof.
\end{proof}

\section{Proof of Proposition \ref{prop: JPhi_kernel_has_dim_L-1}}\label{appendix: Phi_is_regular}

In this appendix we provide a complete proof of Proposition \ref{prop: JPhi_kernel_has_dim_L-1}. However, before we do that, we will need to introduce new notation and two lemmata.

\textbf{Notation.}
 Recall that a polynomial $G$-convolutional neural network of depth $L$, and activation $r$ is the map
    \begin{align*}
        \Phi : \K[G]^L &\to \Sym_\K(n, r^{L-1})[G]\\
        (\theta_1, \hdots, \theta_L) &\mapsto \sigma_r\!\Bigl(
        \cdots 
        \sigma_r\bigl(
           \sigma_r(x \!\conv\! \theta_1)\! \conv\! \theta_2
        \bigr)
        \cdots
        \!\conv\! \theta_{L-1}
     \Bigr)
     \!\conv\! \theta_L.
    \end{align*}
    Equivalently, the map $\Phi$ can be defined recursively by the formula
\[
\Phi(\theta_1, \hdots, \theta_L)
  = \sigma_r\bigl(\Phi(\theta_1,\dots,\theta_{L-1})\bigr)
    \conv \theta_L .
\]
Recall also from Section \ref{sec: pgcnn} that the map $\Phi$ can also be defined in terms of the map $\varphi$ by the formula
\[
\Phi(\theta_1, \dots, \theta_L) = \bigl(x^{\,r^{L-1}} \conv \varphi(\theta_1, \dots, \theta_L)\bigr)\big|_{G}.
\numberthis\label{eq:relation_Phi_varphi}
\]
For $\theta := (\theta_1, \hdots, \theta_L) \in \K[G]^L $, define the filter $\Phi_\theta$ to be the one obtained by evaluating $\Phi$ at $\theta$ i.e.
\begin{align*}
    \Phi_\theta: G &\to \Sym_\K(n, r^{L-1})\\
    g &\mapsto \Phi(\theta)(g).
\end{align*}
 Each monomial in  $\Phi_\theta(g) \in \Sym_\K(n, r^{L-1})$ is indexed by an unordered tuple
\[
[g_1,\dots,g_{r^{L-1}}]\in G^{r^{L-1}},
\]
or, equivalently, by multiplicities
\[
[r_1 g_1,\dots,r_j g_j],
\qquad
r_1+\cdots+r_j = r^{L-1}.
\]
We denote the coefficient of this monomial in \(\Phi_\theta(g)\) by
\[
\Phi_\theta(g)[r_1 g_1,\dots,r_j g_j],
\]
or simply \(\Phi_\theta[r_1 g_1,\dots,r_j g_j]\) when the group element \(g\) is clear from context.

\begin{lemma}\label{lemma :Jacobian_of_Phi_and_relation_between_varphi_Phi}
Let \(\theta=(\theta_1,\dots,\theta_L)\) and \(\dot\theta=(\dot\theta_1,\dots,\dot\theta_L)\), and set 
\[
\theta'=(\theta_1,\dots,\theta_{L-1}),\qquad
\dot\theta'=(\dot\theta_1,\dots,\dot\theta_{L-1}).
\]
Then:
\begin{enumerate}
    \item For any index \([g_1,\dots,g_{r^{L-1}}]\) we have
    \[
    \Phi_\theta(g)[g_1,\dots,g_{r^{L-1}}]
      = \sum_{\delta\in S_{r^{L-1}}}
        \varphi_\theta(
            g_{\delta(1)}^{-1}g,
            \dots,
            g_{\delta(r^{L-1})}^{-1}g
        )
      = \Phi_\theta(e)[g^{-1}g_1,\dots,g^{-1}g_{r^{L-1}}],
    \numberthis\label{eq:coefficients_of_Phi_g}
    \]
    where \(S_m\) denotes the permutation group on \(m\) elements.

    \item As proved in \cite{shahverdi2024geometryoptimizationpolynomialconvolutional},
    \[
    J_\theta\Phi(\dot\theta)
      =
      r\bigl(
         J_{\theta'}\Phi(\dot\theta')
         \circbd
         \sigma_{r-1}(\Phi_{\theta'})
       \bigr) \conv \theta_L
      \;+\;
      \sigma_r(\Phi_{\theta'})\conv \dot\theta_L .
    \]
    The Jacobian of \(\Phi_\theta\) is obtained by computing the Jacobian of each polynomial coefficient of \(\Phi_\theta(g)\) for all \(g\in G\).
\end{enumerate}
\end{lemma}

\begin{proof}
Using \eqref{eq:relation_Phi_varphi}, for any \(g\in G\),
\[
\Phi_\theta(g)
  = \sum_{h\in G^{r^{L-1}}} x^{\,r^{L-1}}(h)\,
    \varphi_\theta(h^{-1}g).
\]
Since \(x^{\,r^{L-1}}(h_1)=x^{\,r^{L-1}}(h_2)\) whenever \(h_1,h_2\) differ by a permutation, the formula \eqref{eq:coefficients_of_Phi_g} follows immediately.
\end{proof}

In other words, the first part of Lemma \ref{lemma :Jacobian_of_Phi_and_relation_between_varphi_Phi} states that the coefficients of the polynomial $\Phi_\theta(g)$ are completely determined by the coefficients of the polynomial $\Phi_\theta(e)$ at the group identity $e\in G$.

\begin{lemma}\label{lemma: jac_Phi_simplified_case}
    Let $\theta = (\theta_1, \hdots, \theta_{L-1}, e) \in \K[G]^L$ where $\theta_i$ is a general filter for $1\leq i < L$ and $e \in G \subset \K[G]$ is the identity element. If $\dot \theta = (\dot \theta_1, \dots, \dot \theta_L) \in \ker J_\theta \Phi$ then $\dot \theta_L = \lambda e$ for some $\lambda \in \K$.
\end{lemma}
\begin{proof}
Let $G=\{g_1=e,g_2,\dots,g_n\}$, $\theta'=(\theta_1,\dots,\theta_{L-1})$ and $\dot\theta'=(\dot\theta_1,\dots,\dot\theta_{L-1})$.
For brevity set
\[
J := J_{\theta'}\Phi(\dot\theta').
\]
Lemma~\ref{lemma :Jacobian_of_Phi_and_relation_between_varphi_Phi}
implies that $\dot\theta\in\ker J_\theta\Phi$ if and only if
\[
rJ \circbd \sigma_{r-1}(\Phi_{\theta'})
= -\,\sigma_r(\Phi_{\theta'})\conv\dot\theta_L.
\numberthis\label{eq:jac_kernel_Phi_at_theta_L_equal_to_e}
\]
Evaluate the filters on both sides in Eq. \eqref{eq:jac_kernel_Phi_at_theta_L_equal_to_e} at $g=e$ to get
    \[
(rJ\circbd\sigma_{r-1}(\Phi_{\theta'}))(e)
=
-\sum_{h\in G}\sigma_r(\Phi_{\theta'})(h)\dot\theta_L(h^{-1}).
\numberthis\label{eq:jac_kernel_Phi_at_e_for_theta_L=e}
\]
Equation ~\eqref{eq:jac_kernel_Phi_at_e_for_theta_L=e} is an equality of two homogeneous polynomials in $\Sym_\K(x, r^{L-1})$. 
Therefore for an index $[r_1g_1,\dots,r_ng_n]$ with $r_1+\dots+r_n=r^{L-1}$,
equation~\eqref{eq:jac_kernel_Phi_at_e_for_theta_L=e} gives the equation

\[
(rJ\circbd\sigma_{r-1}(\Phi_{\theta'}))(e)[r_1g_1,\dots,r_ng_n]
=
-\sum_{h\in G}\sigma_r(\Phi_{\theta'})(h)[r_1g_1,\dots,r_ng_n]\dot\theta_L(h^{-1}),
\]
which by Proposition \ref{lemma :Jacobian_of_Phi_and_relation_between_varphi_Phi} is equivalent to 

\[
(rJ\circbd\sigma_{r-1}(\Phi_{\theta'}))(e)[r_1g_1,\dots,r_ng_n]
=
-\sum_{h\in G}\sigma_r(\Phi_{\theta'})(e)[r_1hg_1,\dots,r_nhg_n]\dot\theta_L(h).
\]
Since the above equation is a relation between the polynomials at only $e$, we remove $e$ from the notation and write 
\[
(rJ\circbd\sigma_{r-1}(\Phi_{\theta'}))[r_1g_1,\dots,r_ng_n]
=
-\sum_{h\in G}\sigma_r(\Phi_{\theta'})[r_1hg_1,\dots,r_nhg_n]\dot\theta_L(h). \tag{$\ast$}\label{eq:jac_kernel_Phi_at_e_for_theta_L=e_for_general_coeff}
\]
These equations \eqref{eq:jac_kernel_Phi_at_e_for_theta_L=e_for_general_coeff} form a system of equations in the unknowns $\dot \theta$; one equation for each multindex  $[r_1g_1,\dots,r_ng_n]$.
Note that on the one hand an unknown appears in the homogeneous polynomial $J$ if and only if it comes from the first $L-1$ filters in $\dot \theta$. On the other hand, the claim of this lemma concerns only the unknowns $\dot \theta_{L}$.
Therefore our goal is to eliminate the unknown homogeneous polynomial \(J\) from the system \eqref{eq:jac_kernel_Phi_at_e_for_theta_L=e_for_general_coeff}. 
To achieve this, we need to express the coefficient of $J$ in terms of $\dot \theta_L$. We do this step by step as follows.

\medskip
\noindent\textbf{First coefficients.}
For the index $[r^{L-1}g_i]$, Eq \eqref{eq:jac_kernel_Phi_at_e_for_theta_L=e_for_general_coeff} becomes
\[
r(J\circbd\sigma_{r-1}(\Phi_{\theta'}))[r^{L-1}g_i]
=
-\sum_{h\in G}\sigma_r(\Phi_{\theta'})[r^{L-1}hg_i]\,\dot\theta_L(h).
\]
Unfolding the Hadamard product on the left hand side gives
\[
r(J\circbd\sigma_{r-1}(\Phi_{\theta'}))[r^{L-1}g_i]
=
rJ[r^{L-2}g_i]\,\sigma_{r-1}(\Phi_{\theta'})[(r-1)r^{L-2}g_i].
\]
So we obtain
\[
J[r^{L-2}g_i]
=
-\sum_{h\in G}
\frac{\sigma_r(\Phi_{\theta'})[r^{L-1}hg_i]}{
r\,\sigma_{r-1}(\Phi_{\theta'})[(r-1)r^{L-2}g_i]}
\dot\theta_L(h).
\tag{$\dagger$}
\label{eq:J_rL2_gi}
\]

\medskip
\noindent\textbf{Next coefficients.}
For the multiindex $[g_1,(r^{L-1}-1)g_i]$,
\[
r(J\circbd\sigma_{r-1}(\Phi_{\theta'}))[g_1,(r^{L-1}-1)g_i]
=
-\sum_{h\in G}
\sigma_r(\Phi_{\theta'})[hg_1,(r^{L-1}-1)hg_i)]\,\dot\theta_L(h).
\]
The left side expands as
\[
rJ[g_1,(r^{L-2}-1)g_i]\,\sigma_{r-1}(\Phi_{\theta'})[(r-1)r^{L-2}g_i]
+
rJ[r^{L-2}g_i]\,\sigma_{r-1}(\Phi_{\theta'})[g_1,((r-1)r^{L-2}-1)g_i],
\]
so together with~\eqref{eq:J_rL2_gi} this expresses
$J[g_1,(r^{L-2}-1)g_i]$ linearly in $\dot\theta_L$.

\medskip
\noindent\textbf{Iterating.}
Define for $0\le m\le r^{L-2}$
\[
O_m^i
=
\sum_{h\in G}
\frac{\sigma_r(\Phi_{\theta'})[mhg_1,(r^{L-1}-m)hg_i]}{
r\,\sigma_{r-1}(\Phi_{\theta'})[(r-1)r^{L-2}g_i]}
\dot\theta_L(h),
\]
and
\[
D_k^i
=
\sigma_{r-1}(\Phi_{\theta'})[kg_1,((r-1)r^{L-2}-k)g_i].
\]
Define $O_m$ to be the square $n\times n$ matrix where its rows are given by $O_m^i$. Namely, the entry $(i, h)$ is given by
\[
\frac{\sigma_r(\Phi_{\theta'})[mhg_1,(r^{L-1}-m)hg_i]}{
r\,\sigma_{r-1}(\Phi_{\theta'})[(r-1)r^{L-2}g_i]}.
\]

For $0\le m<k\le r^{L-2}$ set
\[
C_{m,k}^i
=
\sigma_{r-1}(\Phi_{\theta'})[(k-m)g_1,((r-1)r^{L-2}+m-k)g_i]
= D_{k-m}^i,
\]
and
\[
C_i^{m,k}
=
\sum_{\{\alpha_1=m<\cdots<\alpha_f=k\}}\;
\prod_{j=1}^{f-1} C_{\alpha_j,\alpha_{j+1}}^i,
\qquad
C_i^{k,k}=1.
\]

Induction on $r_1$ yields, for $0\le r_1\le r^{L-2}-1$ and $1<i\le n$:
\[
J[r_1g_1,(r^{L-2}-r_1)g_i]
=
\sum_{m=0}^{r_1} O_m^i\,C_i^{m,r_1}.
\tag{$\clubsuit$}\label{eq:J_r1_expression}
\]

\medskip
\noindent\textbf{Constructing the linear system only in $\dot \theta_L$.}
Using~\eqref{eq:J_r1_expression} and~\eqref{eq:J_rL2_gi}, for each $i\ne1$ the coefficient equation \eqref{eq:jac_kernel_Phi_at_e_for_theta_L=e_for_general_coeff} corresponding to $[r^{L-2}g_1,(r-1)r^{L-2}g_i]$ gives
\[
r\sum_{r_1=0}^{r^{L-2}}
J[r_1g_1,(r^{L-2}-r_1)g_i]\, D_{r^{L-2}-r_1}^i
+
rO_{r^{L-2}}^i D_0^i
=0.
\]
This is linear in $\dot\theta_L$, since each $J[\cdot]$ and $O_m^i$ is linear in $\dot\theta_L$.

These $n-1$ equations (indexed by $i\ne1$) form a $(n-1)\times n$ symbolic matrix in the parameters $\theta'$.
The entry in row $i$ and column $h=g_i^{-1}$ contains the term
\[
\bigl(\Phi_{\theta'}[r^{L-2}g_1])\bigr)^{\,r+1}
\bigl(\Phi_{\theta'}[r^{L-2}g_i]\bigr)^{\,r-1},
\]
which is a polynomial in the parameters $\theta'$ (the term appears in the sum  where $r_1=0$). 
Excluding the entry at the column $h=e$, no other entry in this row contains a term that has the factor  $\bigl(\Phi_{\theta'}[r^{L-2}g_1])\bigr)^{\,r+1}$.

Hence the submatrix omitting the column $h=e$ has a symbolically nonvanishing determinant, so the kernel is one--dimensional and spanned by $(1,0,\dots,0)$. Therefore $\dot\theta_L(h) = 0$ for $h\neq e$. This completes the proof of the lemma.
\end{proof}

We are finally ready to give the proof of Proposition \ref{prop: JPhi_kernel_has_dim_L-1}

\begin{proof}[Proof of Proposition \ref{prop: JPhi_kernel_has_dim_L-1}]
We write perturbations as $\dot\theta = (\dot\theta_1,\dots,\dot\theta_L)$.
Set $\theta'=(\theta_1,\dots,\theta_{L-1})$ and $\dot\theta'=(\dot\theta_1,\dots,\dot\theta_{L-1})$.
We proceed by induction on $L$.

\medskip
\noindent\textbf{Base case $L=1$.}
For $L=1$ we have $\Phi_\theta = x\conv \theta_1$, hence
\[
J_\theta \Phi(\dot\theta)
= J_{\theta_1}(x\conv\theta_1)(\dot\theta_1)
= x\conv \dot\theta_1.
\]
This vanishes iff $\dot\theta_1=0$, so the kernel is trivial, consistent with the proposition.

\medskip
\noindent\textbf{Inductive step.}
Assume the proposition holds for all networks with $L-1$ layers.
We prove it for $L$ layers.

For brevity set
\[
J := J_{\theta'}\Phi(\dot\theta').
\]
Lemma~\ref{lemma :Jacobian_of_Phi_and_relation_between_varphi_Phi}
implies that
\[
\dot\theta\in\ker J_\theta\Phi
\quad\Longleftrightarrow\quad
r\bigl(J \circbd \sigma_{r-1}(\Phi_{\theta'})\bigr)\conv\theta_L
= -\,\sigma_r(\Phi_{\theta'})\conv\dot\theta_L.
\numberthis\label{eq:jac_kernel_Phi}
\]

Let $\dot \theta_L = 0$. 
Since $\theta_L$ is general, so by Corollary \ref{cor: general_filter_not_zero_divisors} we get $\theta_L$ is not a zero divisor, and hence $J = 0$. The required result follows by the inductive hypothesis.

Assume $\dot \theta_L \neq 0$, convolute both sides of \eqref{eq:jac_kernel_Phi} with the inverse of $\theta_L$, which exists because $\theta_L$ is general, and let $\dot{\tilde{\theta}}_L := \dot \theta_L \conv \theta_L^{-1}$ to get 
\begin{equation}\label{eq:jac_kernel_Phi_mult_by_theta_L_inverse}
rJ\circbd \Phi_{\theta'}^{\,r-1}
\;=\;
-\,\Phi_{\theta'}^{\,r}\conv \dot{\tilde{\theta}}_L.
\end{equation}

Since $\theta'$ is general and $\dot \theta_L \neq 0$, then by Lemma \ref{lemma: jac_Phi_simplified_case}, we have $\dot{\tilde{\theta}}_L = \lambda e$ for some nonzero $\lambda \in \K$.
This implies $\dot \theta_L = \lambda \theta_L$, and therefore any $\dot\theta\in\ker J_\theta\Phi$ must have the form
\[
\dot\theta = (\dot\theta_1,\dots,\dot\theta_{L-1},\;\lambda\theta_L).
\]
Observe that the vector
\[
(0,\dots,0,\;-\tfrac{\lambda}{r}\theta_{L-1},\;\lambda\theta_L)
\]
lies in the kernel as well (it is a direct check using the Jacobian composition rule). Subtracting this vector from $\dot\theta$ yields
\[
(\dot\theta_1,\dots,\dot\theta_{L-1}-(-\tfrac{\lambda}{r}\theta_{L-1}),\;0)
\;=\;
(\dot\theta_1,\dots,\dot\theta_{L-1}+\tfrac{\lambda}{r}\theta_{L-1},\;0)
\in \ker J_{\theta}\Phi.
\]
By Eq.~\eqref{eq:jac_kernel_Phi} the vector 
\[(\dot\theta_1,\dots,\dot\theta_{L-1}+\tfrac{\lambda}{r}\theta_{L-1}) \in \ker J_{\theta'}\Phi,
\]
and by the inductive hypothesis, it can be written as a linear combination of
\[
\bigl((\theta_1,-r\theta_2,0,\dots,0),\,
(0,\theta_2,-r\theta_3,0,\dots,0),\,
\dots,\,
(0,\dots,0,\theta_{L-2},-r\theta_{L-1})\bigr).
\]
Adding back the previously subtracted vector
$(0,\dots,0,-\tfrac{\lambda}{r}\theta_{L-1},\lambda\theta_L)$
shows that $\dot\theta$ is a linear combination of the $L-1$ vectors stated in the proposition.
This completes the induction and the proof.

\end{proof}

\section{Proof of Proposition \ref{prop: intersection_of_general_secant_is_trivial}} \label{appendix: proof_proposition_4.8}

In this appendix we provide a full proof of Proposition 4.8.
We introduce a set up of the problem first and then restate the proposition for completeness.

In the following $G$ is a group of fixed size $n$, and $\K$ is an infinite field with $\text{char}\ \K = 0$. 
We also denote by $S_d = \Sym_\K(n, d)$ the set of homogeneous polynomials of degree  $d$ in $n$ variables, and fix $d = r^{L-1}$ where $r\geq 2$ and $L\geq 1$. 
For simplicity too, for filters $\theta=(\theta_1, \dots, \theta_L)$ we denote by $p_i(\theta) := \sigma_r(\Phi_\theta)(g_i)\in S_{rd}$ and the generated affine secant line by $W_\theta = \langle p_1(\theta), \dots, p_n(\theta) \rangle $.
 We also define $\theta^o := (e, \dots, e)$ to be the tuple of $L$ identity filters, then $p_i(\theta^o) = x_i^{rd}$ and $W_{\theta^o} = \langle x_1^{rd}, \dots, x_n^{rd}\rangle$.
 Finally, all dimensions are Krull dimensions of $\K$-schemes and $\operatorname{length} (Y_\theta) = \dim_{\kappa(\theta)} \Gamma(Y_\theta, \mathcal{O}_{Y_\theta})$, where $\kappa(\theta)$ is the residue field at $\theta\in \Theta$.

\begin{proposition}[Proposition \ref{prop: intersection_of_general_secant_is_trivial}]
    For general filters $\theta = (\theta_1, \dots, \theta_L)$, we have
    \[Y_\theta := \mathbb{P}(W_{\theta}) \cap VP_{n, r, d} = \{[p_1(\theta)], \dots, [p_n(\theta)]\} \, .\]
\end{proposition}

The approach we adopt for this proof consists of two main steps. The first step includes showing that the proposition holds for the special filters $\theta^o$; more precisely, we show that the intersection scheme $Y_{\theta^o} = \mathbb{P}(W_{\theta^o}) \cap VP_{n, r, d}$ consists of $n$ reduced points.
In the second step, we propagate this argument to general filters using semi continuity arguments.

\begin{lemma}\label{lemma: pi_r_is_a_closed_immersion}
    The morphism $\pi_r: \PP(S_d) \to \PP(S_{rd}), [q]\mapsto [q^r]$ is a closed immersion.
    Hence $VP_{n, r, d} := \pi_{r}(\PP(S_d)) \cong \PP(S_{d})$ and consequently, for $q \neq 0 \in S_d$ we have 
    \[
    T_{q^r} \widehat{VP}_{n,r,d} = q^{r-1}S_d,
    \]
    where $\widehat{VP}_{n, r, d}$ is the affine cone associated with the projective variety $VP_{n, r, d}$.
\end{lemma}
\begin{proof}
We only show that $\pi_r$ is a closed immersion; the rest is a direct result.
The property of being a closed immersion is fpqc local on the base
\cite[\href{https://stacks.math.columbia.edu/tag/02L6}{Lemma 02L6}]{stacks-project}; which means given a faithfully flat quasi compact (fpqc) covering of the target, the morphism is a closed immersion if and only if the based-change morphism under each element in the cover is a closed immersion.
The morphism $\PP(S_{rd})_{\overline{\K}} \to \PP(S_{rd})_{\K}$ obtained by base change
from $\operatorname{Spec} \overline{\K} \to \operatorname{Spec} \K$ is faithfully flat and
affine, hence a one-element fpqc covering of $\PP(S_{rd})_{\K}$, and the base change of
$\pi_r$ along it is $\pi_{r, \overline{\K}}$. It therefore suffices to treat the case
$\K = \overline{\K}$, which we assume for the rest of the proof.

To show $\pi_r$ is a closed immersion, under the assumption $\K=\overline{\K}$, we rely on a theorem that says a scheme-morphism is a closed immersion if and only if it is proper, universally injective and unramified 
\cite[\href{https://stacks.math.columbia.edu/tag/04XV}{Lemma 04XV}]{stacks-project}.
Next we prove each one of these properties in order.

\emph{$\pi_r$ is proper}. $\PP(S_d)$ is proper over $\K$, $\PP(S_{rd})$ is separated over $\K$, and $\pi_r$ is a morphism over $\K$; hence
$\pi_r$ is proper by 
\cite[\href{https://stacks.math.columbia.edu/tag/01W6}{Lemma 01W6}]{stacks-project}.

\emph{$\pi_r$ is universally injective}. By \cite[\href{https://stacks.math.columbia.edu/tag/01S4}{Lemma 01S4}]{stacks-project}
the morphism $\pi_r$ is universally injective if the extension $\pi_r: \PP(S_d)(L) \to \PP(S_{rd})(L)$ is injective where $L \supset \K$ is a field extension. 
Let $q_1, q_2 \in S_d(L) \subset L[x_g: g\in G]$ such that $q_1^r = \lambda q_2^r$ where $\lambda \in L^*$. Choose $\mu \in \overline{L}$ such that $\mu^r = \lambda$ then $q_1^r = (\mu q_2)^r$ in the unique factorization domain $\overline{L}[x_g: g\in G]$. Hence $q_1 = c q_2$ where $c \in \overline{L}^*$. 
However $q_1 \in L[x_g: g\in G]$ therefore by comparing nonzero coefficients of $q_1$ and $cq_2$ we conclude that $c\in L^*$. Hence, $\pi_r$ is injective, and since $L$ is an arbitrary field extension, it is universally injective. 

\emph{$\pi_r$ is unramified}. At a closed point $[q] \in \PP(S_d)$ we have the differential of $\pi_r$ to be the map 
\[
S_d/\K q \to S_{rd}/\K q^r, \quad v + \K q \mapsto rq^{r-1}v + \K q^r.
\]
The differential is injective because if $rq^{r-1} v = \mu q^r$, where $\mu \in \K$, then $q^{r-1}(rv - \mu q) = 0$, and since $S = \K[x_g: g\in G]$ is a domain then $rv = \mu q$. Hence $v \in \K q$, which makes the differential map injective. 
Since $\K$ is algebraically closed then $\pi_r$ is unramified at the closed point $q$.
Since the ramified locus is closed in $\PP(S_d)$, which is $\K$-scheme of finite type. Therefore if the ramified locus is non-empty then it must include a closed point but since all closed points have been excluded then then the locus is empty and so $\pi_r$ is unramified.
\end{proof}

\begin{lemma}\label{lemma: the_fiber_at_identity_filters_has_length_n}
    For the special filters $\theta^o$ we have $Y_{\theta^o} = \{[p_1(\theta^o)], \dots, [p_n(\theta^o)]\}$, as sets. 
    Moreover, the tangent space of the intersection scheme $Y_{\theta^o}$ vanishes at all points, and so $Y_{\theta^o}$ consists of $n$ reduced points and $\operatorname{length} (Y_{\theta^o}) = n$. 
\end{lemma}
\begin{proof}
    Let $L\supset \K$ be a field extension, and let $[q] \in \mathbb{P}(W_{\theta^o})(L)$ then $q$ is a sum of monomials where each one is is a power of a single variable (there are no monomials of mixed variables such as $x_1 x_2$). 
    Assume $[q] \in VP_{n, r, d}(L)$ then there is $t\in S_d \otimes_\K L$ such that $t^r = \lambda q$ for $\lambda \in L^*$. 
    Note that we worked over the extension $L$ just to show that field extension does not add any point to $Y_{\theta^o}$.
    
    If $t$ is a not power of single variable then $t^r$ contains monomial of mixed variables, due to $\text{char}\ \K = 0$, which contradicts our first observation. Hence $t$ is a pure power and so is $t^r$ and $q$. Hence $Y_{\theta^o} = \{[x_1^{rd}], \dots, [x_n^{rd}]\}$.

    We compute tangent space at $p \in Y_{\theta^o}$ using the formula:

    \[
    T_pY_{\theta^o} = T_p\PP(W_{\theta^o}) \cap T_pVP_{n,r,d}.
    \]
    
    For $p= [x_i^{rd}]$, we have $T_{x_i^{rd}} W_{\theta^o} = W_{\theta^o}$, and by Lemma \ref{lemma: pi_r_is_a_closed_immersion} we have $T_{x_i^{rd}}\widehat{VP}_{n, r, d} = x_i^{d(r-1)} S_d$.
    At the intersection of these tangent spaces we get the  polynomial equation
    \[
    \sum_{j=1}^n c_j x_j^{rd} = x_i^{d(r-1)} q,
    \]
    where $c_j \in \K$ and $q\in S_d$. Since $r > 1$, then setting $x_i = 0$ gives us 
    \[
    \sum_{j=1, j\neq i}^n c_j x_j^{rd} = 0
    \]
    which implies $c_j = 0$ for all $j\neq i$. Therefore,
    $T_{x_i^{rd}}\widehat{VP}_{n, r, d} \cap T_{x_i^{rd}} W_{\theta^o} = <x_i^{rd}>$ which is killed by projectivization. Hence the intersection of the tangent spaces of the corresponding projective varieties is trivial.

    Recall that by definition, the tangent space $T_{[x_i^{rd}]}Y_{\theta^o}$ is the dual to $\mathfrak{m}_{[x_i^{rd}]}/\mathfrak{m}_{[x_i^{rd}]}^2$ where  $\mathfrak{m}_{[x_i^{rd}]}$ is the maximal ideal in the local ring $\mathcal{O}_{Y_{\theta^o}, [x_i^{rd}]}$.
    By the above $\mathfrak{m}_{[x_i^{rd}]}/\mathfrak{m}_{[x_i^{rd}]}^2 = 0$, and by Nakayama Lemma, we get $\mathfrak{m}_{[x_i^{rd}]} = (0)$ is the zero ideal, hence the local ring  $\mathcal{O}_{Y_{\theta^o}, [x_i^{rd}]}$ is a field and therefore $[x_i^{rd}]$ is a reduced point. 
    Each $[x_i^{rd}]$ is a $\K$-rational point, so this field is
$\kappa([x_i^{rd}]) = \K$. Hence
\[
\Gamma(Y_{\theta^o}, \mathcal{O}_{Y_{\theta^o}}) \cong \K^n
\qquad\text{and}\qquad
\operatorname{length}(Y_{\theta^o}) = n .
\]
\end{proof}

\begin{lemma}\label{lemma: the_general_fiber_has_length_n}
    Let $\Theta^o \subset \Theta$ be the locus where $p_1(\theta), \dots, p_n(\theta)$ are linearly independent; note $\theta^o \in \Theta^o$.
    Then there is a dense open set $U\subset \Theta^o$ where for all $\theta \in U$ the intersection scheme $Y_\theta$ consists of $n$ reduced points.
\end{lemma}
\begin{proof}
    In this proof, we define an incidence scheme $Y$ and a projection morphism $\pi:Y\to \Theta^o$, such that the intersection scheme $Y_\theta$ coincide with the fiber scheme of $\pi$ over $\theta \in \Theta^o$. In the first step, we show that $\pi$ is proper. Then in two consecutive steps, relying on lemmata \cite[\href{https://stacks.math.columbia.edu/tag/0D4I}{Lemma 0D4I}]{stacks-project} and \cite[\href{https://stacks.math.columbia.edu/tag/07ZC}{Lemma 07ZC}]{stacks-project}, we show that the subset in $\Theta^o$ where the intersection scheme $Y_\theta$ is of dimension $0$ and has length $n$ is open and dense in $\Theta = \K^{nL}$.
    
    Consider the incidence scheme
    \[
    Y := \{(\theta, [F]) \in \Theta^o \times VP_{n,r,d}\ |\ p_1(\theta) \wedge \dots p_n(\theta) \wedge F = 0 \in \bigwedge^{n+1} S_{rd}\}.
    \]

    The scheme $Y$ is closed in $\Theta^o \times VP_{n,r,d}$ because it is defined by equations that are polynomials in $\theta$ and linear in $F$.
    This implies the inclusion $\iota: Y \hookrightarrow \Theta^o \times VP_{n,r,d}$ is a closed immersion and hence proper.
    Since $VP_{n,r,d}$ is projective then the projection $pr_1$ to $\Theta^o$ is proper.
    The composition of proper morphisms is proper, therefore the projection $\pi = pr_1 \circ \iota: Y \to \Theta^o$ is proper.

    For $\theta \in \Theta^o$ the fiber scheme of $\pi$ is the intersection scheme $Y_\theta := \PP(W_\theta) \cap VP_{n,r,d}$. By Lemma \ref{lemma: the_fiber_at_identity_filters_has_length_n} we have $Y_{\theta^o}$ is of finite length equal to $n$. 
    By \cite[\href{https://stacks.math.columbia.edu/tag/0D4I}{Lemma 0D4I}]{stacks-project} which states that for a proper morphism, the fiber dimension function over the base is upper semi-continuous.
    This implies the subset $B := \{\theta \in \Theta^o\ |\ \dim Y_{\theta} \geq 1\}$ is closed in $Y$. 
    Then $U_1 := \Theta^o \backslash B$ open and non-empty because $Y_{\theta^o}$ is finite by Lemma \ref{lemma: the_fiber_at_identity_filters_has_length_n} and therefore $\theta^o \in U_1$.
     Our next goal is to shrink $U_1$ further so that it only contains parameters where their fibers are finite and of lenght $n$.
     
    The morphism $\pi$ is proper and quasi-finite over $U_1$ therefore it is finite by \cite[\href{https://stacks.math.columbia.edu/tag/02LS}{Lemma 02LS}]{stacks-project}. 
    Finiteness of $\pi$ implies the pushforward sheaf $\mathcal{F} = \pi_* \mathcal{O}_Y$ restricted over $U_1$  is coherent.
    Recall that the fiber $\mathcal{F} \otimes \kappa(\theta)$ of the sheaf $\mathcal{F}$ over the  point $\theta$ is the $\kappa(\theta)$-vector space obtained by evaluating the germs in the stalk $\mathcal{F}_\theta$ over the point $\theta$. 
    The rank of the fiber is the dimension of the vector space.
    Moreover finiteness of $\pi$ over $U_1$ implies it is affine there and therefore the pushforward $\pi_{*}$ commutes with an arbitrary base change \cite[\href{https://stacks.math.columbia.edu/tag/02KG}{Lemma 02KG}]{stacks-project}; in particular it commutes with the base change given by the morphism $\operatorname{Spec} \kappa(\theta) \to U_1$ that corresponds to the inclusion of a parameter $\theta$ in $U_1$.
    This commutativity implies $\mathcal{F}\otimes \kappa(\theta) \cong \dim_{\kappa(\theta)}\Gamma(Y_\theta, O_Y)$. Hence
    
    \[
    \rank(\mathcal{F} \otimes \kappa(\theta)) = \dim_{\kappa(\theta)}\Gamma({\theta}\times Y_\theta, O_Y) = \text{length}\ (Y_\theta). 
    \]

    The rank of a coherent sheaf is upper semi-continuous by \cite[\href{https://stacks.math.columbia.edu/tag/07ZC}{Lemma 07ZC}]{stacks-project}, therefore there is a non-empty open set $\theta^o \in U_2 \subset U_1$ where $\text{length} (Y_\theta)\leq \text{length}(Y_{\theta^o}) = n$ for all $\theta \in U_2$. However, $\text{length}(Y_\theta) \geq n$ because it always includes $\{\ p_1(\theta), \dots, p_n(\theta)\}$ which are linearly independent by definition and therefore distinct. Hence, $\text{length}(Y_\theta) = n$ for all $\theta \in U_2$ which is open and hence dense in $\Theta^o$, which itself is open and dense in $\Theta = \K^{nL}$, and therefore $U_2$ is open and dense in $\Theta$ as required.
\end{proof}

\begin{proof}[Proof of Proposition 4.8]
    Follows directly from Lemma \ref{lemma: the_general_fiber_has_length_n}.
\end{proof}
\end{document}